\documentclass[10pt]{article}
\usepackage[top=1in, bottom=1in, left=0.9in, right=0.9in]{geometry}

\usepackage{custom_macro}

\title{\textbf{Sparsity-Based Interpolation of External, Internal and Swap Regret}}
\author{
Zhou Lu\thanks{Equal contribution, alphabetical order.}\\
Princeton University\\
\texttt{zhoul@princeton.edu}\\
\and
Y. Jennifer Sun\footnotemark[1]\\
Princeton University\\
\texttt{ys7849@princeton.edu}\\
\and
Zhiyu Zhang\footnotemark[1]\\
Harvard University\\
\texttt{zhiyuz@seas.harvard.edu}\\
}
\date{\vspace{-5ex}}

\begin{document}
\maketitle

\begin{abstract}
Focusing on the expert problem in online learning, this paper studies the interpolation of several performance metrics via $\phi$-regret minimization, which measures the total loss of an algorithm by its regret with respect to an arbitrary action modification rule $\phi$. With $d$ experts and $T\gg d$ rounds in total, we present a single algorithm achieving the instance-adaptive $\phi$-regret bound
\begin{equation*}
\tilde O\rpar{\min\left\{\sqrt{d-d^\unif_\phi+1},\sqrt{d-d^\self_\phi}\right\}\cdot\sqrt{T}},
\end{equation*}
where $d^\unif_\phi$ is the maximum amount of experts modified identically by $\phi$, and $d^\self_\phi$ is the amount of experts that $\phi$ trivially modifies to themselves. By recovering the optimal $O(\sqrt{T\log d})$ external regret bound when $d^\unif_\phi=d$, the standard $\tilde O(\sqrt{T})$ internal regret bound when $d^\self_\phi=d-1$ and the optimal $\tilde O(\sqrt{dT})$ swap regret bound in the worst case, we improve upon existing algorithms in the intermediate regimes. In addition, the computational complexity of our algorithm matches that of the standard swap-regret minimization algorithm due to \cite{blum2007external}.

Technically, building on the well-known reduction from $\phi$-regret minimization to external regret minimization on stochastic matrices, our main idea is to further convert the latter to online linear regression using Haar-wavelet-inspired matrix features. Then, by associating the complexity of each $\phi$ instance with its sparsity under the feature representation, we apply techniques from comparator-adaptive online learning to exploit the sparsity in this regression subroutine.
\end{abstract}

\section{Introduction}

We consider the distributional version of \emph{Learning from Expert Advice} (LEA), which is a two-player repeated game between a learning algorithm and an adversary. Let $d\in\N_+$ be the total number of experts. In each ($t$-th) round, the two players interact as follows. 
\begin{enumerate}
\item The learning algorithm picks a distribution $p_t\in\Delta(d)$ over the $d$ experts, where $\Delta(d)\subset\R^d$ denotes the probability simplex, i.e., $\Delta(d)=\{p\in\mathbb{R}^d:p\ge 0, \|p\|_1=1\}$. 
\item The adversary observes $p_t$ and picks a vector $l_t\in[-1,1]^d$ specifying the losses of all experts.
\item The algorithm observes $l_t$ and suffers a loss defined by the inner product $\inner{p_t}{l_t}$.
\item The adversary determines whether the game terminates. If so, let $T$ be the total number of rounds.
\end{enumerate}
Targeting the regime of $T\gg d$, our goal is to design a learning algorithm that suffers low total loss $\sum_{t=1}^T\inner{p_t}{l_t}$, without any additional assumption on the adversary.

It is well-known that such an objective can be approached by minimizing certain comparative performance metrics. Specifically, we study a class of performance metrics called the $\phi$\emph{-regret}, due to \cite{greenwald2003general}. Let $\calS(d)$ be the collection of all linear functions mapping $\Delta(d)$ to itself, which can be equivalently expressed as $d$-by-$d$ right stochastic matrices. With any $\phi\in\calS(d)$ called an \emph{action modification rule}, the $\phi$-regret is defined as
\begin{equation*}
\reg_T(\phi)\defeq\sum_{t=1}^T\inner{p_t}{l_t}-\sum_{t=1}^T\inner{\phi(p_t)}{l_t}. 
\end{equation*}
Intuitively, the $\phi$-regret compares the total loss of the algorithm at the end of the game to the total loss it would have obtained, had it transformed all its past actions according to $\phi$ (while assuming the adversary's past actions remain the same). If an algorithm guarantees that $\reg_T(\phi)\leq f(\phi)$ for some function $f$, then by definition, its total loss can be upper-bounded by the oracle inequality
\begin{equation}\label{eq:oracle}
\sum_{t=1}^T\inner{p_t}{l_t}\leq \min_{\phi\in\calS(d)}\spar{\sum_{t=1}^T\inner{\phi(p_t)}{l_t}+f(\phi)}.
\end{equation}

The $\phi$-regret can be viewed as the interpolation of several important performance metrics:
\begin{itemize}
\item The most common \emph{external regret} (or simply known as ``the'' regret) equals the supremum of $\reg_T(\phi)$ over all the $\phi$'s that are constant functions. The celebrated \emph{Multiplicative Weight Update} (MWU) algorithm \citep{littlestone1994weighted,cesa1997use} achieves the optimal external regret upper bound, $O(\sqrt{T\log d})$. 
\item The \emph{internal regret} \citep{foster1999regret} equals the supremum of $\reg_T(\phi)$ over all the $\phi$'s that map exactly $d-1$ vertices of $\Delta(d)$ to themselves.\footnote{In other words, such a $\phi$ transforms the mass of $p_t$ on one of the experts to the others.} This is often motivated by game-theoretic applications, while \cite{stoltz2005internal} proved its quantitative advantage over the external regret in certain prediction problems. Running MWU over the targeted $\phi$-class guarantees the standard internal regret bound, $O(\sqrt{T\log d})$ \citep[Chapter~4.4]{cesa2006prediction}. 
\item The \emph{swap regret} \citep{blum2007external} is defined as the supremum of $\reg_T(\phi)$ over the entire $\calS(d)$. By construction it is at most $d$ times the internal regret bound, but an even better $O(\sqrt{dT\log d})$ upper bound can be achieved via a well-known, computationally efficient \emph{swap-to-external reduction} \citep{blum2007external}. For the distributional version of LEA, this is recently shown to be optimal in the regime of $T\gg d$ we consider \citep{dagan2024external,peng2024fast}. 
\item The \emph{quantile regret} \citep{chaudhuri2009parameter} compares the total loss of the algorithm to the total loss of the $\ceil{\eps d}$-th best expert. In the language of the $\phi$-regret, this amounts to considering the $\phi$'s that are not only constant, but also close to outputting the uniform distribution. It is known that the optimal upper bound is $O(\sqrt{T\log\eps^{-1}})$ \citep{orabona2016coin,negrea2021minimax}, which recovers the $O(\sqrt{T\log d})$ external regret as its own worst case ($\eps=1/d$).
\end{itemize}

\paragraph{Main result} In this paper, we present a single $\phi$-regret minimization algorithm (Algorithm~\ref{algorithm:main} in Appendix~\ref{section:code}) that ties the above regimes together. Compared to the naive approach of aggregating the above specialized algorithms by MWU, our proposed algorithm achieves a $\phi$-regret bound that depends on suitable complexity measures of each $\phi$ instance, leading to a sharper oracle inequality, Eq.(\ref{eq:oracle}). Specifically, with $e^{(i)}\in\R^d$ representing the canonical basis vector along the $i$-th coordinate, the complexity of $\phi$ is measured by the following two definitions. 

\begin{definition}[Uniformity]\label{definition:uniformity}
The \emph{uniformity} of $\phi\in\calS(d)$, denoted by $d^\unif_\phi$, is defined as the frequency of the most frequent element in the size-$d$ multiset $\{\phi(e^{(1)}),\ldots,\phi(e^{(d)})\}$. 
\end{definition}

\begin{definition}[Degree of self-map]
The \emph{degree of self-map} of $\phi\in\calS(d)$, denoted by $d^\self_\phi$, is defined as the amount of indices $i$ satisfying $\phi(e^{(i)})=e^{(i)}$. 
\end{definition}

In plain words, $d^\unif_\phi$ measures the maximum number of experts modified identically by $\phi$, while $d^\self_\phi$ measures the number of experts modified trivially by $\phi$ (i.e., modified to themselves). One could also interpret $d^\unif_\phi$ as the similarity between $\phi$ and constant functions, while $d^\self_\phi$ represents the similarity between $\phi$ and the \emph{self-map}. In their easiest settings, $d^\unif_\phi=d$ and $d^\self_\phi=d-1$ respectively recover the $\phi$-classes of the external regret and the internal regret, therefore accordingly, one would expect the $\phi$-regret to be decreasing with respect to $d^\unif_\phi$ and $d^\self_\phi$. In the other extreme we have $d^\unif_\phi=1$ and $d^\self_\phi=0$, and a sensible algorithm needs to still guarantee the optimal swap regret bound, $\tilde O(\sqrt{dT})$. 

As a concrete realization of this reasoning, our algorithm guarantees the $\phi$-regret bound
\begin{equation*}
\reg_T(\phi)=\tilde O\rpar{\min\left\{\sqrt{d-d^\unif_\phi+1},\sqrt{d-d^\self_\phi}\right\}\cdot\sqrt{T}},
\end{equation*}
as well as the optimal quantile regret bound $O(\sqrt{T\log\eps^{-1}})$. This corresponds to a natural interpolation of the $O(\sqrt{T\log d})$ external regret, the $\tilde O(\sqrt{T})$ internal regret and the $\tilde O(\sqrt{dT})$ swap regret, only sacrificing (multiplicatively) a constant factor compared to MWU and a polylog factor compared to \citep[Chapter~4.4]{cesa2006prediction} and \citep{blum2007external}, in their respective specialized regimes. 

Finally, from the computational perspective, both the runtime and the memory of our algorithm are of the same order as the algorithm from \citep{blum2007external}. That is, without sacrificing the computational complexity, our algorithm is an instance-adaptive improvement over this classical result on swap regret minimization. 

\paragraph{Techniques} Our algorithm is based on the following idea. It is already known that $\phi$-regret minimization in LEA can be reduced to an external regret minimization problem on the stochastic matrix space $\calS(d)$ \citep{gordon2008no}, for which solutions can be built as variants of mirror descent. Deviating from the latter part of this standard procedure, we apply a linear transform on the unconstrained matrix space $\R^{d\times d}$, converting all the comparing $\phi$'s to their corresponding \emph{transform domain coefficients}. If $\phi$ is sparse on the transform domain, then we see it as ``simple'' -- just like the intuition from Fourier transform, where time series consisting of very few frequencies are considered simple. Then, since learning a hypothesis $\phi$ is equivalent to learning its transform domain coefficient, we can apply a \emph{sparsity-adaptive} online learning algorithm to perform this task, for which there are out-of-the-box options available \citep[Chapter~10]{orabona2025modern}. 

Putting these together, our approach amounts to using a sparsity-adaptive online learning algorithm to solve a particular matrix linear regression subroutine. The crucial step is designing the \emph{features} here (equivalently, the linear transform on the matrix space $\R^{d\times d}$), as we need to ensure its consistency with the inductive bias of our targeted regret bounds. Our construction is based on the \emph{Haar wavelet} \citep{mallat2008wavelet}, whose ability to sparsely represent low-variation signals has enabled several recent advances in online learning \citep{baby2019online,zhang2023unconstrained,jacobsen2024equivalence}. Along the way we address a number of technical challenges to be outlined shortly. 

\subsection{Related Work}

\paragraph{$\Phi$-regret} The $\phi$-regret we study is an instance-dependent version of a better-known concept called the $\Phi$-regret, due to \cite{greenwald2003general}. With respect to any class $\Phi\subset\calS(d)$ of action modification rules, the $\Phi$-regret is defined as the supremum of $\reg_T(\phi)$ over all $\phi\in\Phi$. Due to its generality unifying external, internal and swap regret, further developments have been presented in numerous works afterwards, particularly regarding its connection to various equilibrium concepts in game theory \citep{stoltz2007learning,rakhlin2011online,piliouras2022evolutionary,bernasconi2023constrained,cai2024tractable,zhang2024efficient}. 

Technically, we build on the well-known reduction from swap regret to external regret on the extended domain $\calS(d)$, pioneered by \cite{stoltz2005internal,blum2007external,gordon2008no} and further developed by \cite{ito2020tight}. In a conceptually different manner, one may also analyze the swap regret through the \emph{subsequence regret} \citep{lehrer2003wide,roth2023learning}, and unifying algorithms have been proposed based on certain multi-objective formulations of online learning \citep{lee2022online,haghtalab2023calibrated,noarov2023high} related to the $L_\infty$-norm \emph{Blackwell approachability} of the negative orthant \citep{blackwell1956analog,perchet2015exponential,shimkin2016online}. The key idea here is to convert the subsequence regret to the regret of a ``meta'' LEA algorithm that reweighs different subsequences. In particular, utilizing the meta algorithm of \cite{chen2021impossible}, the approach of \cite{roth2023learning} can achieve an instance-dependent refinement of the $\tilde O(\sqrt{T})$ internal regret bound (see Appendix~\ref{section:discussion}). Generic reductions between Blackwell approachability and no-regret learning have also been studied by \cite{abernethy2011blackwell,dann2023pseudonorm,dann2024rate}.

A recent breakthrough of \cite{dagan2024external} and \cite{peng2024fast} showed that the $d$-dependence of the time-averaged $\tilde O(\sqrt{d/T})$ swap regret bound can be improved exponentially, at the price of an exponentially worse dependence on $T$. This is orthogonal to the regime of $T\gg d$ we consider, but very intriguingly, their algorithms are also based on some sort of multi-resolution analysis. Closer to our setting, they showed that the $\tilde O(\sqrt{dT})$ swap regret is optimal in the distributional version of LEA with $T\gg d$, answering an open problem from earlier works. 

\paragraph{Adaptive online learning} This work aligns with the general theme of adaptivity in online learning, e.g., \citep{mcmahan2014unconstrained,foster2015adaptive,cutkosky2018algorithms,mhammedi2020lipschitz}, whose main idea is to achieve regret bounds that scale with a priori unknown complexity measures of the problem instance. Existing approaches can be roughly categorized into two types: the first type uses a generic expert meta algorithm to aggregate multiple copies of a ``non-adaptive'' base algorithm (e.g., mirror descent with fixed learning rate) with different hyperparameter setups \citep{foster2017parameter,chen2021impossible}, while the second type uses specialized designs to eschew the expert aggregation \citep{duchi2011adaptive,luo2015achieving,zhang2024improving,cutkosky2024fully}, trading generality for computational efficiency. Our approach belongs to the second type. 

\paragraph{Features and sparsity} Our techniques are inspired by recent advances in \emph{dynamic regret} minimization, which is the hardest regret notion that compares to an arbitrary sequence of predictions selected in hindsight. \cite{zhang2023unconstrained} presented a reduction from dynamic regret to external regret on the extended domain $\R^T$, which is reminiscent of the swap-to-external reduction discussed above; also see \citep{jacobsen2024equivalence}. Central to this technique is the use of features to associate the complexity of a hypothesis to the sparsity of its linear representations. This suggests viewing the features of \citep{blum2007external} as the canonical matrix basis, and our main conceptual takeaway is that the Haar wavelet, introduced to online learning by \cite{baby2019online}, produces matrix features that are better-aligned with the inherent structure of the external and internal regret. 

\paragraph{Technical challenges} We now highlight the technical challenges of this work along the lines of the above discussion. Regardless of the choice of features, the dynamic regret bound of \cite{zhang2023unconstrained} is $\tilde O(\sqrt{T}\cdot\sqrt{T})$ in the worst case, where one of the $\sqrt{T}$ is the iconic rate of online linear optimization, and the other one comes from the dimensionality of the extended domain $\R^T$. By default, it means the analogous approach for our problem would result in the suboptimal $\tilde O(\sqrt{d^2}\cdot\sqrt{T})$ swap regret. The key subtlety here is that the surrogate gradients from the swap-to-external reduction are well-structured, such that one could use \emph{``first-order'' gradient-adaptivity} to further shave off a $\sqrt{d}$ factor \citep{blum2007external}. But this does not hold true for arbitrary features anymore. 

Regarding this issue, we show that the Haar wavelet is particularly ``compatible'' with the gradient structure from the swap-to-external reduction, such that the optimal $\tilde O(\sqrt{dT})$ swap regret bound can still be recovered after incorporating specific Haar-wavelet-type matrix features. Our approach also involves a complexity-preserving augmentation of the hypothesis class $\calS(d)$ (Section~\ref{subsection:preprocessing}), as well as a projection technique that enforces the constraint $\calS(d)$ (Section~\ref{subsection:constraint}). Both are nontrivial constructions absent from \citep{zhang2023unconstrained}. 

\subsection{Notation}

Throughout this paper, $\inner{x}{y}$ denotes the inner product of generic $x$ and $y$. Specifically when acting on two matrices, it equals the Euclidean inner product of their vectorizations. $x\otimes y$ denotes the outer product of vectors. An interval (of integers) $[a:b]$ is defined as $\{a,a+1,\ldots,b-1,b\}$. $\log$ denotes the natural logarithm if the base is omitted. $\bm{I}_d$ represents the $d$-by-$d$ identity matrix. We use the following indexing rule: on a matrix $x$, $x_i$ represents its $i$-th row while $x_{i,j}$ represents its ($i,j$)-th entry; on a vector $x$, $x_i$ represents its $i$-th entry. 

\section{Technical Ingredients}\label{section:ingredients}

The starting point of our approach is the celebrated $\phi$-to-external reduction of \cite{gordon2008no}. This is remarkably simple, with similar ideas also presented in several works around the same time \citep{stoltz2005internal,blum2007external}. First, let us assume access to a computational oracle which, given any input $\phi\in\calS(d)$, returns a fixed point $p\in\Delta(d)$ satisfying $p=\phi(p)$. The existence of $p$ is due to Brouwer's fixed-point theorem, and in practice, it can be obtained using $O(d^3)$ time and $O(d^2)$ memory via a generic linear system solver.\footnote{In theory, the asymptotic time complexity of exact fixed-point computation can be reduced via fast matrix multiplication, although the associated algorithms are somewhat less practical. There are also approaches that approximate the fixed point instead, e.g., via the power iteration \citep{greenwald2008more,mohri2017online}; also see \citep[Appendix~I.3]{zhang2024efficient}.} 

Now, regarding the LEA problem, suppose we have some procedure that generates a linear map $\phi_t\in\calS(d)$ in each round. By querying the fixed point oracle, our LEA algorithm simply predicts its output $p_t$ such that $p_t=\phi_t(p_t)$. The instantaneous regret of this algorithm with respect to any comparing action modification rule $\phi^*$ can be expressed as
\begin{equation*}
\inner{p_t-\phi^*(p_t)}{l_t}=\inner{\phi_t(p_t)-\phi^*(p_t)}{l_t}=\inner{p_t\otimes l_t}{\phi_t-\phi^*},
\end{equation*}
where the last step interprets $\phi_t$ and $\phi^*$ as right stochastic matrices. Then, by using the notation $g_t\defeq p_t\otimes l_t$ and taking the summation over $t\in[1:T]$, the $\phi$-regret can be rewritten as
\begin{equation}\label{eq:ggm}
\reg_T(\phi^*)=\sum_{t=1}^T\inner{g_t}{\phi_t-\phi^*}.
\end{equation}
Notice that the RHS is the external regret of our $\phi_t$-predicting procedure in a hypothetical \emph{Online Linear Optimization} (OLO) problem: in each round it predicts $\phi_t\in\calS(d)$, observes the rank-$1$ loss gradient $g_t\in\R^{d\times d}$, and suffers the linear loss $\inner{g_t}{\phi_t}$. The total loss is compared to that of $\phi^*$.

Given this reduction, the convention of the field is to proceed with \emph{Online Mirror Descent} (OMD). For example, by running a separate copy of MWU on each row of $\calS(d)$ (which corresponds to OMD with a ``row-separable'' Bregman divergence), one could obtain an $O(d\sqrt{T\log d})$ swap regret bound via a straightforward summation. A better approach is using a \emph{gradient-adaptive} OMD algorithm to exploit the structure of $g_t$, leading to the optimal $O(\sqrt{dT\log d})$ swap regret bound \citep{blum2007external}. This could be seen as a computationally efficient realization of running MWU on all vertices of $\calS(d)$, which also intuitively justifies the necessity of the $\sqrt{d}$ factor. 

In this paper we deviate from this convention. Instead of using OMD, our algorithm (Algorithm~\ref{algorithm:main}) will be a modular composition of several new components introduced next. 

\subsection{Preprocessing}\label{subsection:preprocessing}

As the first step, we need to assign an ordering to the set of experts and augment their indices into a sequence of dyadic length (i.e., power of $2$). This is due to our later use of the Haar wavelet: the size of the Haar matrix is naturally dyadic, and the ordering of its input affects our intermediate, variation-based regret bound (Theorem~\ref{theorem:external}). 

Concretely, we define $S\defeq\lceil\log_2 d\rceil$ and $\bar d\defeq 2^S$, which means that $\bar d$ is the smallest power of $2$ that is greater or equal to $d$, the total number of experts. To associate the \emph{augmented index set} $[1:\bar d]$ with the original index set $[1:d]$, our algorithm requires a \emph{relabeling function} $\calI$ as the user's input, mapping any $i\in[1:d]$ to a set $\calI(i)\subset[1:\bar d]$. Three conditions should be satisfied: 
\begin{itemize}
\item For all $i\in[1:d]$, $\calI(i)$ only contains consecutive integers in $[1:\bar d]$.
\item For all $i\in[1:d]$, the cardinality of $\calI(i)$ satisfies $1\leq \abs{\calI(i)}\leq 2$.
\item The collection of sets $\{\calI(i);i\in[1:d]\}$ is disjoint, and $\cup_{i=1}^d\calI(i)=[1:\bar d]$. 
\end{itemize}
It is clear that such an $\calI$ exists.\footnote{A simple construction is to keep the ordering of the original $d$ experts unchanged, and make a duplicate for each expert (placed next to it) until the total number of expert is $\bar d$.} Besides, the third condition enables the definition of the ``inverse'' function $\calI^{-1}:[1:\bar d]\rightarrow[1:d]$, such that for all $\bar i\in[1:\bar d]$ we have $\bar i\in\calI(\calI^{-1}(\bar i))$.

Using this construction, we then consider the following \emph{augmented} LEA problem with $\bar d$ experts. Let $\bar l_t\in\R^{\bar d}$ be the vector whose $\bar i$-th entry equals the $\calI^{-1}(\bar i)$-th entry of $l_t$, and any decision $\bar p_t\in\Delta(\bar d)$ for the augmented LEA problem suffers the loss $\inner{\bar p_t}{\bar l_t}$. The point is that such an augmented problem can be made equivalent to the original one: 
\begin{itemize}
\item By letting the $i$-th coordinate of $p_t\in\Delta(d)$ be the total mass of $\bar p_t$ on $\calI(i)$, we always have $\inner{p_t}{l_t}=\inner{\bar p_t}{\bar l_t}$. 
\item For any $\phi^*\in\calS(d)$, there exists $\bar\phi^*\in\calS(\bar d)$ that preserves the complexity of $\phi^*$ and satisfies $\inner{\phi^*(p_t)}{l_t}=\inner{\bar\phi^*(\bar p_t)}{\bar l_t}$ (Lemma~\ref{lemma:relabeling_uniform} and \ref{lemma:relabeling_self}). 
\end{itemize}
Therefore based on the reduction of \cite{gordon2008no}, it suffices to consider external regret minimization on the \emph{augmented stochastic matrix space} $\calS(\bar d)$ instead. 

To proceed, we use bar-equipped notations analogous to Eq.(\ref{eq:ggm}): let $\bar \phi_t$ be our matrix prediction on $\calS(\bar d)$, with $\bar p_t\in\Delta(\bar d)$ being its fixed point. Then, with $\bar g_t\defeq \bar p_t\otimes \bar l_t$, our goal next is to design a $\bar \phi_t$-predicting procedure which ensures that for all $\bar \phi^*\in\calS(\bar d)$, $\sum_{t=1}^T\inner{\bar g_t}{\bar \phi_t-\bar \phi^*}$ is upper-bounded by an appropriate function of $\bar \phi^*$. 

\subsection{Wavelet-Inspired Matrix Features}\label{subsection:wavelet}

Our second step requires pretending that \emph{improper} matrix predictions are allowed, i.e., $\bar \phi_t$ can be anything on $\R^{\bar d\times \bar d}$ rather than just $\calS(\bar d)$. Although its fixed point does not exist in general (which is problematic from the perspective of \cite{gordon2008no}), we will show afterwards that the constraint $\calS(\bar d)$ can be enforced through an additional projection step (Section~\ref{subsection:constraint}). 

This leads to the main idea of this work: given any collection of matrices $\calB$ that spans $\R^{\bar d\times\bar d}$, all elements of $\R^{\bar d\times\bar d}$ can be expressed as a linear combination of $\calB$, which then facilitates the use of sparsity to measure the complexity of each comparator instance $\bar\phi^*$. There are just two desiderata when choosing $\calB$: first, it should properly associate different matrix entries such that the resulting external and internal regret is low; and second, it needs to be congruent with the structure of the gradient $\bar g_t$, such that the $\tilde O(\sqrt{dT})$ swap regret is preserved in the worst case. We will show that the classical idea of \emph{Haar wavelet} \citep{mallat2008wavelet} addresses both. 

Specialized to Euclidean spaces, wavelets refer to a systematic construction of orthogonal \emph{multi-resolution} bases, with numerous applications across signal processing, nonparametric statistics and machine learning. Here multi-resolution means that if we examine the inner product of any ``data'' vector with all wavelet basis vectors (i.e., taking the \emph{wavelet transform}), then some of these inner products (i.e., transform domain coefficients) capture the global characteristics of the data, while others capture the local characteristics. Such a representation is extremely useful for common modalities such as images: a typical image would consists of global scenes, local details and pixel-level noises, therefore separately extracting these information would enable applications like denoising and compression. Our algorithm exemplifies the same intuition in an online context: extracting multi-resolution characteristics of the gradient $\bar g_t$ enables better decision making. 

\paragraph{Haar wavelet} Concretely, we use the Haar wavelet basis $\calH$ defined as follows. Recall that our augmented dimensionality satisfies $\bar d=2^S$ for some positive integer $S$. Given any \emph{scale} parameter $s\in[1:S]$ and any \emph{location} parameter $l\in[1:2^{S-s}]$, we define a Haar basis vector $h^{(s,l)}\in\R^{\bar d}$ whose $i$-th entry is
\begin{equation*}
h^{(s,l)}_i\defeq\begin{cases}
1,& i\in[2^s(l-1)+1: 2^s(l-1)+2^{s-1}];\\
-1,& i\in[2^s(l-1)+2^{s-1}+1: 2^sl];\\
0,&\textrm{else}.
\end{cases}
\end{equation*}
Notice that $h^{(s,l)}$ is only nonzero on its \emph{support} $I^{(s,l)}\defeq[2^s(l-1)+1:2^sl]$; on the first half of $I^{(s,l)}$ it equals $1$, and on the second half of $I^{(s,l)}$ it equals $-1$. For each $s$ there are $2^{S-s}$ valid choices of $l$, which means that in total there are $\bar d-1$ pairs of $(s,l)$. Then, we additionally define $h^{(S,0)}$ as the all-one vector in $\R^{\bar d}$, which completes the size-$\bar d$ collection of vectors,
\begin{equation*}
\calH\defeq\left\{h^{(S,0)}\right\}\cup\left\{h^{(s,l)};s\in[1:S], l\in[1:2^{S-s}]\right\}.
\end{equation*}

One could verify that $\calH$ is indeed an orthogonal basis of $\R^{\bar d}$: the all-one vector $h^{(S,0)}$ is orthogonal to all $h^{(s,l)}$; vectors on the same scale have disjoint supports; and for two vectors on different scales, the larger-scale vector remains constant over the support of the smaller-scale vector. Consequently, any vector $v\in\R^{\bar d}$ can be uniquely represented by all the inner products $\inner{v}{h}$, $\forall h\in\calH$. The most important property for our use is that $\inner{v}{h}$ captures the \emph{variability} of $v$ over the support of $h$: if $v$ remains constant over such an interval, then $\inner{v}{h}=0$. In other words, the Haar wavelet basis can sparsely represent low-variation signals. 

\paragraph{Matrix features} Starting from $\calH$, we construct the following collection $\calB$ of \emph{matrix features}, and this will be used in an online linear regression subroutine introduced shortly. Let $\calE=\{e^{(i)}; i\in[1:\bar d]\}$ be the collection of $\bar d$-dimensional unit coordinate vectors. We define
\begin{equation*}
\calB\defeq \left\{\bm{I}_{\bar d}\right\}\cup\left\{h\otimes e; h\in\calH, e\in\calE\right\},
\end{equation*}
where $\bm{I}_{\bar d}$ denotes the $\bar d$-by-$\bar d$ identity matrix. Without $\bm{I}_{\bar d}$, the rest of $\calB$ is an orthogonal basis of the matrix space $\R^{\bar d\times\bar d}$, which can sparsely represent any comparator $\bar\phi^*\in \calS(\bar d)$ with large $d^\unif_{\bar\phi^*}$. The role of $\bm{I}_{\bar d}$ is to help in cases with large $d^\self_{\bar\phi^*}$.

To summarize, we will work with this slightly \emph{overcomplete} collection of matrix features. It means representations are not unique, but this is fine -- as long as there \emph{exists} a sparse representation for the comparator $\bar\phi^*$, i.e., $\bar\phi^*=\sum_{b\in\calB}\Phi^{*,(b)}b$ for some sparse coefficients $\{\Phi^{*,(b)};b\in\calB\}$, we can use the known algorithms introduced next to adapt to it. 

\subsection{First-Order Sparsity-Adaptive Oracle}

Continuing from our improper matrix prediction problem on $\R^{\bar d\times\bar d}$ (the beginning of Section~\ref{subsection:wavelet}; predictions and loss gradients are denoted by $\bar\phi^\improper_t$ and $\bar g^\improper_t$), our third step is to solve it by learning $\{\Phi^{*,(b)};b\in\calB\}$, the representation of the comparator $\bar\phi^*$ on $\calB$. Since $\bar\phi^*$ can be arbitrary, the ``learning'' here more precisely means suffering low regret in the following equivalent problem ``transformed'' by $\calB$. In each round we first predict a coefficient $\Phi_t^{(b)}\in\R$ for all matrix features $b\in\calB$, which results in an improper matrix prediction
\begin{equation}\label{eq:improper}
\bar\phi_t^{\mathrm{improper}}\defeq\sum_{b\in\calB}\Phi_t^{(b)}b\in\R^{\bar d\times\bar d}.
\end{equation}
After observing the corresponding loss gradient $\bar g_t^\improper\in\R^{\bar d\times\bar d}$, we suffer the linear loss $\inner{\bar g_t^\improper}{\bar\phi_t^\improper}=\sum_{b\in\calB}\inner{\bar g_t^\improper}{b}\Phi_t^{(b)}$. The previous matrix-based regret definition $\sum_{t=1}^T\inner{\bar g_t^\improper}{\bar \phi_t^\improper-\bar \phi^*}$ is equivalent to the sum of a one-dimensional regret,
\begin{equation*}
\sum_{b\in\calB}\sum_{t=1}^T\inner{\bar g_t^\improper}{b}(\Phi_t^{(b)}-\Phi^{*,(b)}).
\end{equation*}
In other words, the problem of improper matrix prediction is converted to \emph{online linear regression with linear losses}.

Suppose we simply run gradient descent on each coefficient $\Phi_t^{(b)}$ separately. With the optimal constant learning rate, it guarantees the well-known ``prototypical'' bound $\sum_{t=1}^T\inner{\bar g_t^\improper}{b}(\Phi_t^{(b)}-\Phi^{*,(b)})=O\rpar{\abs{\Phi^{*,(b)}}\sqrt{\sum_{t=1}^T\inner{\bar g_t^\improper}{b}^2}}$ (see, e.g., \citep[Theorem~2.13]{orabona2025modern}), and therefore
\begin{equation}\label{eq:regret_decomposition}
\sum_{t=1}^T\inner{\bar g_t^\improper}{\bar \phi^\improper_t-\bar \phi^*}=O\rpar{\sum_{b\in\calB}\abs{\Phi^{*,(b)}}\sqrt{\sum_{t=1}^T\inner{\bar g_t^\improper}{b}^2}}.
\end{equation}
This is essentially what we need as the RHS depends on the number of nonzero elements within $\{\Phi^{*,(b)};b\in\calB\}$. The only issue (but critical) is the lack of \emph{adaptivity}: the required constant learning rate on each $\Phi_t^{(b)}$ would depend on $\Phi^{*,(b)}$ (more specifically, its absolute value), but we need a single algorithm that simultaneously guarantees the desirable one-dimensional regret bound for all possible values of $\Phi^{*,(b)}$. No ``oracle tuning'' is allowed. 

We address this issue using \emph{comparator-adaptive} OLO algorithms \citep{streeter2012no,mcmahan2014unconstrained,orabona2016coin}, for which an excellent tutorial is given by \citep[Chapter~10]{orabona2025modern}. The idea is that by replacing the one-dimensional gradient descent by suitable instances of \emph{Follow the Regularized Leader} (FTRL), we can bypass the choice of the learning rate, thus concretely achieving a suitable weaker form of the above regret bounds. Specifically, the one-dimensional algorithm we use comes from the following lemma combining \citep[Theorem~4]{zhang2022pde} and \citep[Theorem~5.8]{cutkosky2018algorithms}. 

\begin{lemma}[\citep{cutkosky2018algorithms,zhang2022pde}]\label{lemma:one_d_alg}
Consider the one-dimensional OLO problem where in each round an algorithm makes a decision $x_t\in\R$ and then observes the loss gradient $c_t\in [-G,G]$ for some known Lipschitz constant $G$. Given any hyperparameter $\eps>0$, there exists an algorithm that guarantees
\begin{equation*}
\sum_{t=1}^Tc_t(x_t-u)\leq \sqrt{G^2+G\sum_{t=1}^T\abs{c_t}}\spar{2\eps+2\sqrt{2}\abs{u}\sqrt{\log\rpar{1+\frac{\abs{u}}{\sqrt{2}\eps}}}+4\sqrt{2}\abs{u}},
\end{equation*}
for all time horizon $T\in\N_+$, all comparator $u\in\R$ and all possible $c_{1:T}$ sequence. 
\end{lemma}

We remark that \citep[Theorem~4]{zhang2022pde} is the state-of-the-art result for comparator-adaptive OLO \emph{without} gradient adaptivity, while \citep[Theorem~5.8]{cutkosky2018algorithms} enhances it to \emph{first-order gradient adaptivity} in a black-box manner (i.e., $G^2T$ in the regret bound is improved to $G^2+G\sum_{t=1}^T\abs{c_t}$, following the notations of Lemma~\ref{lemma:one_d_alg}). As shown by \cite{blum2007external}, first-order gradient adaptivity is crucial for achieving the optimal $\tilde O(\sqrt{dT})$ swap regret bound. Our analysis additionally shows that such gradient adaptivity works harmoniously with the Haar wavelet, such that the $\tilde O(\sqrt{dT})$ swap regret bound can still be achieved with suitable matrix features.

\subsection{Enforcing the Constraint}\label{subsection:constraint}

Our final step is to close the only remaining gap in the above reasoning: $\bar\phi_t^\improper$ does not belong to $\calS(\bar d)$. Addressing this issue requires a suitable wrapper: given the algorithm from the previous step equipped with an upper bound on $\sum_{t=1}^T\inner{\bar g_t^\improper}{\bar \phi^\improper_t-\bar \phi^*}$, we design mappings $\bar \phi^\improper_t\rightarrow \bar \phi_t\in\calS(\bar d)$ called a \emph{projection oracle} and $\bar g_t\rightarrow\bar g_t^\improper$ called a \emph{gradient processing oracle}, such that (Lemma~\ref{lemma:regret_wrapper})
\begin{equation*}
\sum_{t=1}^T\inner{\bar g_t}{\bar \phi_t-\bar \phi^*}\leq \sum_{t=1}^T\inner{\bar g_t^\improper}{\bar \phi^\improper_t-\bar \phi^*}. 
\end{equation*}
The LHS is equivalent to $\reg_T(\phi^*)$, while the RHS adapts to the complexity of $\phi^*$ (measured with respect to $\calB$). Furthermore, the processed gradient $\bar g_t^\improper$ should be low enough in a suitable notion of magnitude (Lemma~\ref{lemma:gradient_norm}). 

To this end, we employ the following two-stage procedure that extends a classical OLO-to-LEA wrapper\footnote{\cite{luo2015achieving,orabona2016coin} converted the problem of LEA (i.e., OLO on the domain $\Delta(d)$) to ``unconstrained'' OLO on the domain $\R^d$. A more general treatment is given by \cite{cutkosky2018black}.} \citep{luo2015achieving,orabona2016coin} to the matrix setting. 
\begin{align}
\bar \phi^\improper_t\in\R^{\bar d\times\bar d}\xrightarrow[]{\quad\text{Stage 1}\quad}\bar \phi^+_t\in\R_+^{\bar d\times\bar d}\xrightarrow[]{\quad\text{Stage 2}\quad}&~\bar \phi_t\in\calS(\bar d)\tag{Projection}\\
&\Downarrow\nonumber\\
\bar g^\improper_t\xleftarrow[]{\quad\text{Stage 1}\quad}\bar g^+_t\xleftarrow[]{\quad\text{Stage 2}\quad} & ~\bar g_t\tag{Gradient processing}
\end{align}
The first stage enforces the positivity constraint by mapping $\bar \phi^\improper_t$ to an intermediate prediction $\bar \phi^+_t\in\R_+^{\bar d\times\bar d}$, while the second stage enforces the stochastic matrix constraint by mapping $\bar \phi^+_t$ to $\bar \phi_t$. Accordingly, the gradient processing is also performed in two stages, $\bar g_t\rightarrow\bar g^+_t\rightarrow\bar g_t^\improper$. The details are deferred to Appendix~\ref{subsection:detail_constraint}. 

With that we have presented all the technical ingredients of our algorithm, as well as the high-level analytical strategy. The next step is to combine all the pieces into our main results. 

\section{Main Result}\label{section:main}

Our main results are Algorithm~\ref{algorithm:main} and its $\phi$-regret bound (Theorem~\ref{theorem:main}). The pseudocode is deferred to Appendix~\ref{section:code}, while Figure~\ref{fig:algorithm} illustrates its main idea. We note that since $\abs{\calB}=\Theta(d^2)$, our algorithm uses $O(d^2)$ memory and $O(d^2)+\textsc{FP}_d$ time per round, where $\textsc{FP}_d$ denotes the time complexity of the fixed-point oracle. As $\textsc{FP}_d$ dominates $O(d^2)$ in general, the computational bottleneck is the fixed-point computation. These exactly match the standard swap-regret minimization algorithm from \citep{blum2007external}. 

Another remark is that although our algorithm requires the relabeling function $\calI$ as the user's input (Section~\ref{subsection:preprocessing}), all results in this section hold uniformly for all $\calI$ (i.e., the multiplying constants $c$ in the theorems do not depend on $\calI$). The role of good $\calI$ is discussed in Remark~\ref{remark:relabeling}. 

\begin{figure}[ht]
    \centering
    \includegraphics[width=0.8\linewidth]{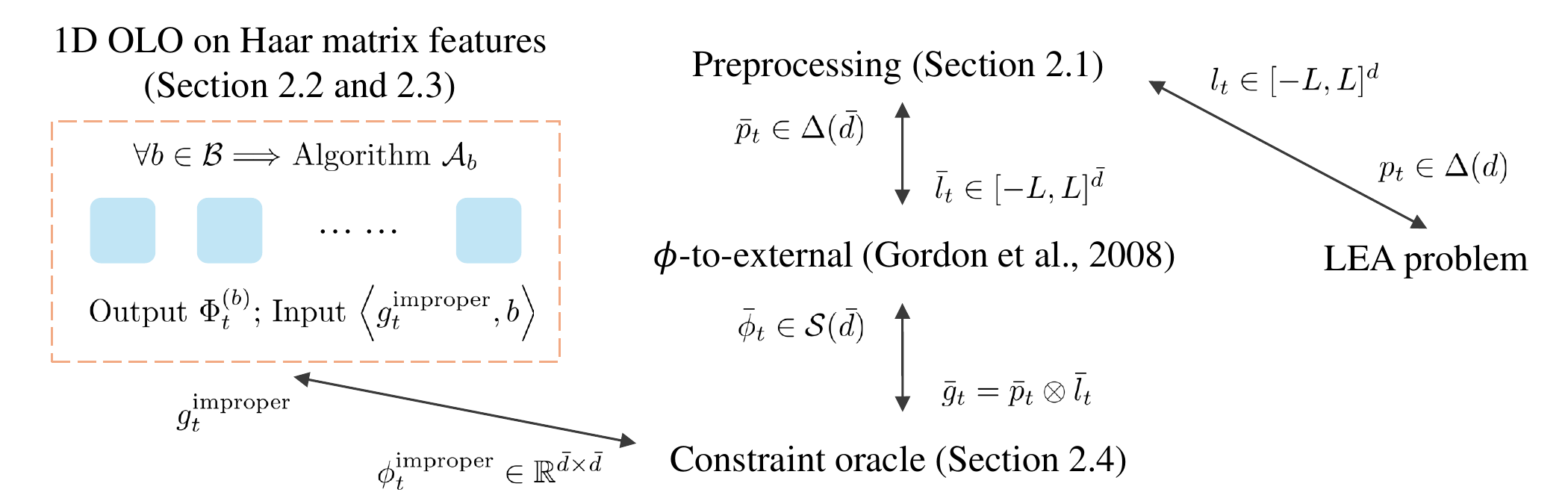}
    \caption{An illustration of Algorithm~\ref{algorithm:main}; see Appendix~\ref{section:code} for details. }
    \label{fig:algorithm}
\end{figure}

\begin{theorem}[Main result]\label{theorem:main}
There is an absolute constant $c>0$ such that for all $T\in\N_+$ and all comparing action modification rule $\phi^*\in\calS(d)$, Algorithm~\ref{algorithm:main} guarantees
\begin{equation*}
\reg_T(\phi^*)\leq c\cdot\rpar{\sqrt{\min\left\{d-d^\unif_{\phi^*}+1,d-d^\self_{\phi^*}\right\}(T+d)}\cdot(\log d)^{3/2}}.
\end{equation*}
\end{theorem}

Assuming $T\gg d$, Theorem~\ref{theorem:main} recovers the $\tilde O(\sqrt{T})$ external regret in the case of $d^\unif_{\phi^*}=d$, the $\tilde O(\sqrt{T})$ internal regret in the case of $d^\self_{\phi^*}=d-1$ and the $\tilde O(\sqrt{dT})$ swap regret in the worst case, matching the results achieved by specialized algorithms (modulo polylog factors). Although this basic requirement can also be achieved by simply aggregating specialized algorithms using MWU,\footnote{One could simultaneously maintain three algorithms targeting the external, internal and swap regret respectively, treat them as ``meta''-experts, and aggregate them using a three-expert MWU algorithm on top. The weakness is that this baseline does not achieve a better-than-$\tilde O(\sqrt{dT})$ $\phi$-regret bound with respect to a generic $\phi$.} concrete improvements of our approach are 
manifested in the intermediate regimes. For example, suppose $d^\unif_{\phi^*}=d-k+1$, then Algorithm~\ref{algorithm:main} guarantees $\reg_T(\phi^*)=\tilde O(\sqrt{kT})$ while being computationally efficient and agnostic to $k$. This improves both
\begin{itemize}
\item generic swap regret minimization algorithms achieving $\tilde O(\sqrt{dT})$ \citep{blum2007external}; and
\item the computationally inefficient, $k$-dependent baseline which runs MWU over an appropriate $\phi$-class (specifically, all zero-one stochastic matrices $\phi$ satisfying $d^\unif_{\phi}=d-k+1$). MWU can guarantee the $\tilde O(\sqrt{kT})$ $\phi$-regret despite the cardinality of the hypothesis class being exponential in $k$, but computationally this is impractical when $k$ is large. 
\end{itemize}
Similar cases of improvement can be constructed for large $d^\self_{\phi^*}$ as well, although it requires a more subtle comparison to related works \citep{roth2023learning,noarov2023high}. This is deferred to Appendix~\ref{section:discussion}. 

\paragraph{Proof sketch of Theorem~\ref{theorem:main}} For simplicity, suppose Eq.(\ref{eq:regret_decomposition}) can indeed be achieved (the rigorous analysis using Lemma~\ref{lemma:one_d_alg} is similar). Combining the components from Section~\ref{section:ingredients}, the LHS of Eq.(\ref{eq:regret_decomposition}) upper-bounds $\reg_T(\phi^*)$. The RHS of Eq.(\ref{eq:regret_decomposition}) is a summation over $b\in\calB$, and excluding the edge cases we may only consider those $b$ expressed as $h^{(s,l)}\otimes e^{(j)}$, where $s\in[1:S]$ is the scale parameter, $l\in[1:2^{S-s}]$ is the location parameter, and $j\in[1:\bar d]$ is the column index. The remaining task becomes upper-bounding
\begin{equation}\label{eq:sketch}
\sum_{s=1}^S\sum_{l=1}^{2^{S-s}}\sum_{j=1}^{\bar d}\abs{\Phi^{*,(s,l,j)}}\sqrt{\sum_{t=1}^T\inner{\bar g_t^\improper}{h^{(s,l)}\otimes e^{(j)}}^2}.
\end{equation}
More precisely, the ``transform domain coefficient'' $\Phi^{*,(s,l,j)}\in\R$ here is defined such that $\Phi^{*,(s,l,j)}\cdot h^{(s,l)}\otimes e^{(j)}$ is the projection of either $\bar\phi^*$ or $\bar\phi^*-\bm{I}_{\bar d}$ to $h^{(s,l)}\otimes e^{(j)}$. Recall that such non-uniqueness is due to $\calB$ being overcomplete. 

To proceed, there are two crucial steps. 
\begin{itemize}
\item Lemma~\ref{lemma:gradient_norm} shows that $\abs{\inner{\bar g_t^\improper}{h^{(s,l)}\otimes e^{(j)}}}\leq 2\sum_{i\in I^{(s,l)}}\bar p_{t,i}$, where $I^{(s,l)}\subset[1:\bar d]$ is the support of $h^{(s,l)}$. Since on each row the sum of $\bar\phi^*$ entries is $1$, we also have $\sum_{j=1}^{\bar d}\abs{\Phi^{*,(s,l,j)}}\leq \mathrm{const}$ (Lemma~\ref{lemma:coefficient_magnitude}). It means even very crudely,
\begin{multline*}
\mathrm{Eq.(\ref{eq:sketch})}\lesssim \sum_{s=1}^S\sum_{l=1}^{2^{S-s}}\sum_{j=1}^{\bar d}\abs{\Phi^{*,(s,l,j)}}\sqrt{\sum_{t=1}^T\sum_{i\in I^{(s,l)}}\bar p_{t,i}}\lesssim \sum_{s=1}^S\sum_{l=1}^{2^{S-s}}\sqrt{\sum_{t=1}^T\sum_{i\in I^{(s,l)}}\bar p_{t,i}}\\
\overset{\mathrm{Cauchy-Schwarz}}{\lesssim} \sum_{s=1}^S\sqrt{2^{S-s}}\sqrt{\sum_{t=1}^T\sum_{l=1}^{2^{S-s}}\sum_{i\in I^{(s,l)}}\bar p_{t,i}}\overset{\bar p_t\in\Delta(\bar d)}{\lesssim} \sum_{s=1}^S\sqrt{2^{S-s}T}\overset{S\lesssim\log d}{\lesssim} \sqrt{dT}\log d,
\end{multline*}
which recovers the nonadaptive, $\tilde O(\sqrt{dT})$ swap regret bound. 
\item The adaptive bound requires a sparsity-based refinement. Let us reconsider $\Phi^{*,(s,l,j)}$ projected from $\bar\phi^*$: if for all $i\in I^{(s,l)}$ the $i$-th row of $\bar\phi^*$ equals the same vector, then due to the property of the Haar wavelet, $\sum_{j=1}^{\bar d}\abs{\Phi^{*,(s,l,j)}}=0$. Now the question is, for any fixed $s$, how many different $l$ make this sum nonzero? By definition, the answer is at most $\mathrm{RowSwitch}(\bar\phi^*)\defeq\sum_{i=1}^{\bar d-1}\bm{1}[\bar \phi^*_{i}\neq \bar \phi^*_{i+1}]$, which means the same Cauchy-Schwarz argument gives us the $\tilde O\rpar{\sqrt{\mathrm{RowSwitch}(\bar\phi^*)\cdot T}}$ bound rather than $\tilde O(\sqrt{dT})$. The last step is to construct the augmented comparator $\bar\phi^*$ in a complexity-preserving manner, i.e., $\mathrm{RowSwitch}(\bar\phi^*)\lesssim d-d^\unif_{\phi^*}$ (Lemma~\ref{lemma:relabeling_uniform}). 

The proof of the $d^\self_{\phi^*}$-dependent bound is similar. The key observation is that $\mathrm{RowSwitch}(\bar\phi^*-\bm{I}_{\bar d})$ is small if $d^\self_{\bar \phi^*}$ is large. Since everything in this proof concerns the \emph{same algorithm} (just with different constructions of $\bar\phi^*$ and $\Phi^{*,(s,l,j)}$), the $d^\unif_{\phi^*}$ and $d^\self_{\phi^*}$-dependent bounds are achieved simultaneously. 
\end{itemize}

\begin{remark}[Role of relabeling]\label{remark:relabeling}
Compared to typical LEA algorithms, Algorithm~\ref{algorithm:main} has a notable difference: the algorithm generates different outputs under a permutation of the expert indices. Theorem~\ref{theorem:main} proves that regardless of the permutation its $\phi$-regret bound improves \citep{blum2007external}, but as shown in the proof sketch, this is actually a \emph{relaxation} of an order-dependent result (Theorem~\ref{theorem:external}). In this regard, the relabeling function $\calI$ could be seen as the user's prior: if it renders $\mathrm{RowSwitch}(\bar\phi^*)$ small, then the order-dependent bound would further improve Theorem~\ref{theorem:main}. 

As an example, consider $d=\bar d$, with the first $d/2$ rows of $\phi^*$ being some vector $v$ while the rest of the rows being $w\neq v$. Such a $\phi^*$ itself possesses certain simplicity, and we suppose the user happens to preserve this simplicity by ``trivially'' picking $\calI(i)=i;\forall i\in[1:d]$, which ensures $\bar\phi^*=\phi^*$ by the construction of Lemma~\ref{lemma:relabeling_uniform}. In this case, although Theorem~\ref{theorem:main} guarantees the same $\reg_T(\phi^*)=\tilde O(\sqrt{dT})$ as \citep{blum2007external}, Theorem~\ref{theorem:external} sharpens it to $\tilde O\rpar{\sqrt{\mathrm{RowSwitch}(\bar\phi^*)\cdot T}}=\tilde O(\sqrt{T})$. We expect that reaping such benefits in ``ordered'' downstream problems, such as the full swap regret problem \citep{fishelson2025full} where experts are embedded into $\R^n$ thus equipped with natural orderings, could be an intriguing direction for future works. 
\end{remark}

Finally, we can characterize the quantile regret of Algorithm~\ref{algorithm:main}, which is formally defined as follows. Imagine that at the end of the game the cumulative loss of all experts ($\sum_{t=1}^Tl_{t,i};\forall i\in[1:d]$) are sorted from the lowest to the highest, with ties broken arbitrarily. Then, for any $\eps\in[d^{-1},1]$, let $i_\eps$ be the index of the $\ceil{\eps d}$-th element in this sorted list. The $\eps$-quantile regret is defined as
\begin{equation*}
\reg_T(\eps)\defeq \sum_{t=1}^T\inner{p_t}{l_t}-\sum_{t=1}^Tl_{t,i_\eps}.
\end{equation*}

\begin{restatable}{theorem}{quantile}\label{theorem:quantile}
There exists an absolute constant $c>0$ such that for all $\eps\in[d^{-1}:1]$, Algorithm~\ref{algorithm:main} guarantees the quantile regret bound
\begin{equation*}
\reg_T(\eps)\leq c\sqrt{T\log \eps^{-1}}.
\end{equation*}
\end{restatable}

This is optimal \citep{negrea2021minimax}, and importantly, it sharpens the $\tilde O(\sqrt{T})$ external regret bound from Theorem~\ref{theorem:main} to the optimal rate $O(\sqrt{T\log d})$, including the log factor. The proof hinges on the observation that the quantile regret corresponds to a special case of the $\phi$-regret where the rows of the comparing $\phi$ matrix are all equal. It means we may only analyze a subset of $\calB$ generated by the all-one feature $h^{(S,0)}\in\calH$, and the rest of the proof is fairly standard \citep[Chapter~10.6]{orabona2025modern}. 

\section{Conclusion}

Focusing on the LEA problem with $T\gg d$, this paper leverages the ideas of features and sparsity to refine the classical reduction approaches in $\phi$-regret and swap regret minimization \citep{blum2007external,gordon2008no}. We present an LEA algorithm with an adaptive $\phi$-regret upper bound, matching the external, internal and swap regret of specialized algorithms while improving them in the intermediate regimes. Conceptually, our key observation is that incorporating certain Haar-wavelet-inspired matrix features can introduce inductive biases that are well-aligned with important special cases of the $\phi$-regret. Technically, we show that the Haar wavelet is naturally congruent with the benign gradient structure from the $\phi$-to-external reduction, such that the optimal $\tilde O(\sqrt{dT})$ swap regret bound can still be achieved after incorporating such matrix features. We leave a thorough study of the downstream benefits to future works. In addition, the readers are referred to \citep{hait2025comparator} for technical improvements on this topic, as well as a concrete application to learning in games. 

\section*{Acknowledgment}

We thank the reviewers for their valuable feedback. ZL and JS thank Elad Hazan for the financial support. ZZ thanks Heng Yang for the financial support through Harvard University Dean’s Competitive Fund for Promising Scholarship.

\bibliography{phi_wavelet}

\newpage
\section*{Appendix}
\appendix

Appendix~\ref{section:code} presents the pseudocode of Algorithm~\ref{algorithm:main} alongside the algorithmic details omitted in the main paper. To analyze this algorithm, Appendix~\ref{section:detail_ingredients} presents several lemmas on its technical ingredients introduced in Section~\ref{section:ingredients}. These are then used in Appendix~\ref{section:proof_main} to prove our main results. Appendix~\ref{section:discussion} contains detailed comparison to existing results. 

\section{Details of the Algorithm}\label{section:code}

\subsection{Constraint Oracle}\label{subsection:detail_constraint}

Our approach to enforce the constraint $\calS(d)$ (Section~\ref{subsection:constraint}) is described below. 

\paragraph{Stage 1: $\R^{\bar d\times\bar d}\rightarrow \R_+^{\bar d\times\bar d}$.} Starting from $\bar \phi^\improper_t$, we define the intermediate prediction $\bar \phi^+_t$ whose $(i,j)$-th entry $\bar \phi^+_{t,i,j}$ equals $\max\left\{\bar \phi^\improper_{t,i,j}, 0\right\}$, the maximum of the $(i,j)$-th entry of $\bar \phi^\improper_t$ and $0$. This is performed for all $i,j\in[1:\bar d]$, resulting in a positive matrix $\bar \phi^+_t\in\R^{\bar d\times\bar d}_+$. 

Regarding gradient processing, given the intermediate gradient $\bar g_t^+\in\R^{\bar d\times\bar d}$, we define the $(i,j)$-th entry of $\bar g^\improper_t$ as
\begin{equation*}
\bar g^\improper_{t,i,j}\defeq\begin{cases}
\bar g^+_{t,i,j},& \bar \phi^+_{t,i,j}=\bar \phi^\improper_{t,i,j},\\
\min\left\{\bar g^+_{t,i,j},0\right\},& \mathrm{else}.
\end{cases}
\end{equation*}
The intuition is the following. On each entry,
\begin{itemize}
\item If the ``unprojected prediction'' $\bar \phi^\improper_{t,i,j}$ is already positive, then no projection is needed since $\bar \phi^+_{t,i,j}=\max\left\{\bar \phi^\improper_{t,i,j}, 0\right\}=\bar \phi^\improper_{t,i,j}$. Naturally the gradient is kept unchanged, $\bar g^\improper_{t,i,j}=\bar g^+_{t,i,j}$. 
\item If $\bar \phi^\improper_{t,i,j}$ is negative, then it is projected to $\bar \phi^+_{t,i,j}=0$, which means ``positive predictions are favored''. To encourage that we only send negative $\bar g^\improper_{t,i,j}$ to the wrapped algorithm that generates $\bar \phi^\improper_{t,i,j}$. 
\end{itemize}

\paragraph{Stage 2: $\R_+^{\bar d\times\bar d}\rightarrow\calS(\bar d)$.} Given $\bar \phi^+_t\in\R^{\bar d\times\bar d}_+$ from Stage 1, we scale its $i$-th row $\bar\phi^+_{t,i}$ by $1/\norm{\bar\phi^+_{t,i}}_1$, for all $i$. If $\norm{\bar\phi^+_{t,i}}_1=0$ then the scaled vector is defined as $[1/d,\ldots,1/d]$. The matrix after the scaling is defined as $\bar\phi_t\in\calS(\bar d)$. 

Regarding gradient processing, given $\bar g_t\in\R^{\bar d\times\bar d}$, we define the intermediate gradient $\bar g_t^+\in\R^{\bar d\times\bar d}$ whose $i$-th row is
\begin{equation*}
\bar g^+_{t,i}\defeq\bar g_{t,i}-\inner{\bar g_{t,i}}{\bar\phi_{t,i}}.
\end{equation*}
Here $\bar\phi_{t,i}\in\Delta(\bar d)$ denotes the $i$-th row of the proper matrix prediction $\bar\phi_{t}$. The intuition is that $\bar g^+_{t,i}$ represents the ``centered'' version of $\bar g_{t,i}$ with respect to $\bar\phi_{t}$. The obtained intermediate gradient $\bar g_t^+$ is then passed to Stage 1. 

\subsection{Pseudocode}

The pseudocode of our algorithm is presented as Algorithm~\ref{algorithm:main}. It only requires the relabeling function $\calI$ as the user's input. 

\begin{algorithm*}[!ht]
\caption{Wavelet-based LEA with adaptive $\phi$-regret guarantee.\label{algorithm:main}}
\begin{algorithmic}[1]
\REQUIRE A user-specified relabeling function $\calI$ (Section~\ref{subsection:preprocessing}). Besides, recall the several components introduced previously: 
\begin{itemize}[leftmargin=*,itemsep=0pt,topsep=5pt]
\item A fixed point oracle (the beginning of Section~\ref{section:ingredients}). 
\item The collection $\calB$ of matrix features (Section~\ref{subsection:wavelet}).
\item The one-dimensional OLO algorithm from Lemma~\ref{lemma:one_d_alg}, denoted by $\A_\oned$. It requires a Lipschitz constant $G$ and a positive constant $\eps$ as hyperparameters. 
\item The projection oracle and the gradient processing oracle (Section~\ref{subsection:constraint} and Appendix~\ref{subsection:detail_constraint}).
\end{itemize}
\STATE For all $b\in\calB$, initiate a copy of $\calA_{\oned}$, denoted by $\calA_{\oned}^{(b)}$. We set $G=2$ for all $\calA_{\oned}^{(b)}$. The choice of $\eps$ depends on $b$, which we denote as $\eps_b$: 
\begin{itemize}
\item If $b=\bm{I}_{\bar d}$, then $\eps_b=1$.
\item For other $b\in\calB$, we can write $b=h\otimes e^{(j)}$ for some $h\in\calH$ and $e^{(j)}\in\calE$. If $h=h^{(S,0)}$ then $\eps_b=\bar d(d^2\abs{\calI(\calI^{-1}(j))})^{-1}$; otherwise $\eps_b=(d^2\abs{\calI(\calI^{-1}(j))})^{-1}$. 
\end{itemize}
\FOR{$t=1,2,\ldots,$}
\STATE For all $b\in\calB$, query the $t$-th prediction $\Phi^{(b)}_t\in\R$ of $\calA_{1d}^{(b)}$. 
\STATE Following Eq.(\ref{eq:improper}), let
\begin{equation*}
\bar\phi_t^{\mathrm{improper}}=\sum_{b\in\calB}\Phi_t^{(b)}b\in\R^{\bar d\times\bar d}.
\end{equation*}
\STATE Using the projection oracle, compute $\bar \phi^\improper_t\rightarrow \bar \phi_t\in\calS(\bar d)$. 
\STATE Using the fixed point oracle, compute $\bar p_t\in\Delta(\bar d)$ such that $\bar p_t=\bar\phi_t(\bar p_t)$.
\STATE Using the relabeling function, compute $p_t\in\Delta(d)$ whose $i$-th entry equals the total mass of $\bar p_t$ on $\calI(i)$, i.e.,
\begin{equation*}
p_{t,i}=\sum_{\bar i\in\calI(i)}\bar p_{t,\bar i}, \quad \forall i\in[1:d].
\end{equation*}
\STATE Output $p_t$ as the $t$-th decision in LEA, and observe the loss vector $l_t\in\R^d$. 
\STATE Using the relabeling function, define $\bar l_t\in\R^{\bar d}$ whose $\bar i$-th entry equals the $\calI^{-1}(\bar i)$-th entry of $l_t$, i.e., 
\begin{equation*}
\bar l_{t,\bar i}=l_{t,\calI^{-1}(\bar i)},\quad\forall \bar i\in[1:\bar d].
\end{equation*}
\STATE Let $\bar g_t=\bar p_t\otimes \bar l_t\in\R^{\bar d\times\bar d}$. 
\STATE Using the gradient processing oracle (which depends on $\bar \phi^\improper_t$), compute $\bar g_t\rightarrow\bar g_t^\improper\in\R^{\bar d\times\bar d}$.
\STATE For all $b\in\calB$, return $\inner{\bar g_t^\improper}{b}$ as the $t$-th loss gradient to $\calA_{\oned}^{(b)}$.
\ENDFOR
\end{algorithmic}
\end{algorithm*}

\section{Details of the Technical Ingredients}\label{section:detail_ingredients}

The subsections below provide the detailed analysis of the algorithmic components from Section~\ref{section:ingredients}, following the order there.

\subsection{Preprocessing}

Based on Section~\ref{subsection:preprocessing}, since this subsection only handles the comparing action modification rules, we will remove the superscript star from the notations $\phi^*$ and $\bar\phi^*$ for conciseness. 

By construction, Algorithm~\ref{algorithm:main} satisfies the loss condition $\inner{p_t}{l_t}=\inner{\bar p_t}{\bar l_t}$. We now show that for any comparing action modification rule $\phi\in\calS(d)$, there exists $\bar \phi\in\calS(\bar d)$ that preserves the complexity of $\phi$ and further satisfies $\inner{\phi(p_t)}{l_t}=\inner{\bar \phi(\bar p_t)}{\bar l_t}$. This means that the $\phi$-regret of Algorithm~\ref{algorithm:main} can be rewritten as $\sum_{t=1}^T\inner{p_t-\phi(p_t)}{l_t}=\sum_{t=1}^T\inner{\bar p_t-\bar \phi(\bar p_t)}{\bar l_t}$, and it suffices to upper-bound the RHS as a higher dimensional problem. 

First, we use $d-d^\unif_\phi$ to measure the complexity of $\phi$. We note that the following lemma critically uses the requirement that for all $i$, the set $\calI(i)$ only contains consecutive integers. The cardinality condition $1\leq\abs{\calI(i)}\leq 2$ is not used. 

\begin{lemma}\label{lemma:relabeling_uniform}
For any $\phi\in\calS(d)$, there exists $\bar\phi\in\calS(\bar d)$ such that the following two conditions hold. 
\begin{itemize}
\item For all $i\in[1:\bar d]$, let $\bar\phi_{i}$ represent the $i$-th row of $\bar \phi$, and let $\bm{1}[\cdot]$ be the indicator function. Then, 
\begin{equation*}
\sum_{i=1}^{\bar d-1}\bm{1}[\bar \phi_{i}\neq \bar \phi_{i+1}]\leq 2 (d-d^\unif_\phi).
\end{equation*}
\item In Algorithm~\ref{algorithm:main} we have $\inner{\phi(p_t)}{l_t}=\inner{\bar \phi(\bar p_t)}{\bar l_t}$.
\end{itemize}
\end{lemma}

\begin{proof}[Proof of Lemma~\ref{lemma:relabeling_uniform}]
We construct $\bar\phi$ entrywise as the following. For all $i,j\in[1:\bar d]$,
\begin{equation*}
\bar\phi_{i,j}=\frac{\phi_{\calI^{-1}(i),\calI^{-1}(j)}}{\abs{\calI(\calI^{-1}(j))}},
\end{equation*}
where the denominator denotes the cardinality of the set $\calI(\calI^{-1}(j))$. It can be verified that
\begin{equation*}
\sum_{j=1}^{\bar d}\bar\phi_{i,j}=\sum_{j=1}^{\bar d}\frac{\phi_{\calI^{-1}(i),\calI^{-1}(j)}}{\abs{\calI(\calI^{-1}(j))}}=\sum_{\calI^{-1}(j)=1}^{d}\phi_{\calI^{-1}(i),\calI^{-1}(j)}=1,
\end{equation*}
therefore $\bar\phi\in\calS(\bar d)$. 

For later use, consider the size-$d$ multiset consisting of all the rows of $\phi$: we write $k^*\in[1:d]$ as the index of an arbitrary element that has the highest frequency in this multiset (ties are broken arbitrarily). Since for any generic $k\in[1:d]$ and any $i,j\in\calI(k)$, the $i$-th row and the $j$-th row of $\bar\phi$ are exactly the same, we denote this shared row as the vector $\bar\phi_{\calI(k)}\in\R^{\bar d}$. In particular, $\bar\phi_{\calI(k^*)}$ has the intuitive interpretation of a ``frequent row''. 

Now, consider the ``row-variational'' quantity $\sum_{i=1}^{\bar d-1}\bm{1}[\bar \phi_{i}\neq \bar \phi_{i+1}]$. Due to the requirements on the relabeling function $\calI$, we can decompose the augmented index set $[1:\bar d]$ into $d$ segments, each corresponding to $\calI(k)$ for some different $k\in[1:d]$. For any $i\in[1:\bar d-1]$ such that $i,i+1\in\calI(k)$, $\bar \phi_{i}=\bar \phi_{i+1}=\bar\phi_{\calI(k)}$ thus $\bm{1}[\bar \phi_{i}\neq \bar \phi_{i+1}]=0$. Therefore if we define $\mathrm{Switch}\subset[1:\bar d-1]$ as the collection of all indices $i\in[1:d-1]$ such that $i$ and $i+1$ belong to different segments, then $\sum_{i=1}^{\bar d-1}\bm{1}[\bar \phi_{i}\neq \bar \phi_{i+1}]=\sum_{i\in\mathrm{Switch}}\bm{1}[\bar \phi_{i}\neq \bar \phi_{i+1}]$. In plain words, the row-variation of $\bar\phi$ only depends on the variation at the edge of segments. 

Notice that for all $i\in[1:\bar d-1]$, $\bm{1}[\bar \phi_{i}\neq \bar \phi_{i+1}]\leq \bm{1}[\bar \phi_{i}\neq \bar\phi_{\calI(k^*)}]+\bm{1}[\bar \phi_{i+1}\neq \bar\phi_{\calI(k^*)}]$. Therefore we can further upper-bound $\sum_{i\in\mathrm{Switch}}\bm{1}[\bar \phi_{i}\neq \bar \phi_{i+1}]$ ``segment-wise'', by $2\sum_{k=1}^d\bm{1}[\bar \phi_{\calI(k)}\neq \bar \phi_{\calI(k^*)}]$. Due to Definition~\ref{definition:uniformity}, there are $d-d^\unif_\phi$ different values of $k$ such that $\bar \phi_{\calI(k)}\neq \bar \phi_{\calI(k^*)}$. Combining above proves that $\sum_{i=1}^{\bar d-1}\bm{1}[\bar \phi_{i}\neq \bar \phi_{i+1}]\leq 2 (d-d^\unif_\phi)$.

Finally we verify the second condition in the lemma. 
\begin{multline*}
\inner{\bar \phi(\bar p_t)}{\bar l_t}=\sum_{i=1}^{\bar d}\bar p_{t,i}\sum_{j=1}^{\bar d}\bar l_{t,j}\frac{\phi_{\calI^{-1}(i),\calI^{-1}(j)}}{\abs{\calI(\calI^{-1}(j))}}=\sum_{i=1}^{\bar d}\bar p_{t,i}\sum_{j=1}^{\bar d}l_{t,\calI^{-1}(j)}\frac{\phi_{\calI^{-1}(i),\calI^{-1}(j)}}{\abs{\calI(\calI^{-1}(j))}}\\
=\sum_{i=1}^{\bar d}\bar p_{t,i}\sum_{\calI^{-1}(j)=1}^{d}l_{t,\calI^{-1}(j)}\phi_{\calI^{-1}(i),\calI^{-1}(j)}=\sum_{\calI^{-1}(j)=1}^{d}l_{t,\calI^{-1}(j)}\sum_{i=1}^{\bar d}\bar p_{t,i}\phi_{\calI^{-1}(i),\calI^{-1}(j)}\\
=\sum_{\calI^{-1}(j)=1}^{d}l_{t,\calI^{-1}(j)}\sum_{\calI^{-1}(i)=1}^{d}\rpar{\sum_{n\in\calI(\calI^{-1}(i))}\bar p_{t,n}}\phi_{\calI^{-1}(i),\calI^{-1}(j)}\\
=\sum_{\calI^{-1}(j)=1}^{d}l_{t,\calI^{-1}(j)}\sum_{\calI^{-1}(i)=1}^{d}p_{t,\calI^{-1}(i)}\phi_{\calI^{-1}(i),\calI^{-1}(j)}=\sum_{i=1}^{d}\sum_{j=1}^{d}p_{t,i}\phi_{i,j}l_{t,j}=\inner{\phi(p_t)}{l_t}.\qedhere
\end{multline*}
\end{proof}

The following lemma constructs $\bar\phi$ which preserves a different notion of complexity based on $d^\self_\phi$. Different from the previous lemma, we use the cardinality condition $1\leq\abs{\calI(i)}\leq 2$ rather than the requirement that $\calI(i)$ only contains consecutive integers. 

\begin{lemma}\label{lemma:relabeling_self}
For any $\phi\in\calS(d)$, there exists $\bar\phi\in\calS(\bar d)$ such that the following two conditions hold. 
\begin{itemize}
\item $\bar d-d^\self_{\bar\phi}\leq 2(d-d^\self_\phi)$.
\item In Algorithm~\ref{algorithm:main} we have $\inner{\phi(p_t)}{l_t}=\inner{\bar \phi(\bar p_t)}{\bar l_t}$.
\end{itemize}
\end{lemma}

\begin{proof}[Proof of Lemma~\ref{lemma:relabeling_self}]
Again we construct $\bar\phi$ entry-wise. Define the collection of ``important indices'' $\calI^*\subset[1:\bar d]$ as $\{\min\calI(k);k\in[1:d]\}$, and by construction $\abs{\calI^*}=d$. For all $i,j\in[1:\bar d]$, we define the ($i,j$)-th entry of $\bar\phi$ as
\begin{equation*}
\bar\phi_{i,j}=\begin{cases}
\phi_{\calI^{-1}(i),\calI^{-1}(j)},& j=i;\\
\phi_{\calI^{-1}(i),\calI^{-1}(j)},& j\notin\calI(\calI^{-1}(i)),j\in\calI^*;\\
0,& \textrm{else}.
\end{cases}
\end{equation*}
It is simple to verify that $\bar\phi\in\calS(\bar d)$. 

Recall that $e^{(i)}$ denotes the unit vector along the $i$-th coordinate, whose dimensionality depends on the context. For all $i\in[1:\bar d]$, we now consider $\bar\phi_i$, the $i$-th row of $\bar\phi$. If $\phi_{\calI^{-1}(i)}=e^{(\calI^{-1}(i))}$, then $\bar\phi_{i,i}=\phi_{\calI^{-1}(i),\calI^{-1}(i)}=1$, which means $\bar\phi_i=e^{(i)}$. Since there are $d-d^\self_\phi$ different values of $k\in[1:d]$ such that $\phi_k\neq e^{(k)}$, we can further use the cardinality condition on the relabeling function $\calI$ to show that there are at most $2(d-d^\self_\phi)$ different values of $i\in[1:\bar d]$ such that $\bar\phi_i\neq e^{(i)}$. Equivalently, $\bar d-d^\self_{\bar\phi}\leq 2(d-d^\self_\phi)$.

Next we verify the second condition in the lemma. 
\begin{align*}
\inner{\bar \phi(\bar p_t)}{\bar l_t}&=\sum_{i=1}^{\bar d}\sum_{j=1}^{\bar d}\bar p_{t,i}\bar\phi_{i,j}\bar l_{t,j}\\
&=\sum_{i=1}^{\bar d}\bar p_{t,i}\rpar{\phi_{\calI^{-1}(i),\calI^{-1}(i)}l_{t,\calI^{-1}(i)}+\sum_{j\notin\calI(\calI^{-1}(i)),j\in\calI^*}\phi_{\calI^{-1}(i),\calI^{-1}(j)}l_{t,\calI^{-1}(j)}}\\
&=\sum_{i=1}^{\bar d}\bar p_{t,i}\rpar{\phi_{\calI^{-1}(i),\calI^{-1}(i)}l_{t,\calI^{-1}(i)}+\sum_{\calI^{-1}(j)\neq\calI^{-1}(i)}\phi_{\calI^{-1}(i),\calI^{-1}(j)}l_{t,\calI^{-1}(j)}}\\
&=\sum_{i=1}^{\bar d}\bar p_{t,i}\sum_{\calI^{-1}(j)=1}^d\phi_{\calI^{-1}(i),\calI^{-1}(j)}l_{t,\calI^{-1}(j)}\\
&=\sum_{\calI^{-1}(i)=1}^{d}\rpar{\sum_{n\in\calI(\calI^{-1}(i))}\bar p_{t,n}}\sum_{\calI^{-1}(j)=1}^d\phi_{\calI^{-1}(i),\calI^{-1}(j)}l_{t,\calI^{-1}(j)}\\
&=\sum_{\calI^{-1}(j)=1}^{d}\sum_{\calI^{-1}(i)=1}^{d}p_{t,\calI^{-1}(i)}\phi_{\calI^{-1}(i),\calI^{-1}(j)}l_{t,\calI^{-1}(j)}\\
&=\sum_{i=1}^{d}\sum_{j=1}^{d}p_{t,i}\phi_{i,j}l_{t,j}\\
&=\inner{\phi(p_t)}{l_t}.\qedhere
\end{align*}
\end{proof}

\subsection{Matrix Features}\label{subsection:detail_matrix_feature}

Based on Section~\ref{subsection:wavelet}, we now consider the representation of an arbitrary augmented action modification rule $\bar\phi^*\in\calS(\bar d)$ using $\calB$, i.e., the choice of 
$\Phi^{*,(b)}$ such that $\bar\phi^*=\sum_{b\in\calB}\Phi^{*,(b)}b$. Specifically, we upper-bound $\abs{\Phi^{*,(b)}}$ for two different representations defined as follows, which appears in the proof of Theorem~\ref{theorem:main}.
\begin{itemize}
\item \textbf{Representation 1.}\quad If $b=\bm{I}_{\bar d}$, then $\Phi^{*,(b)}=0$. For all $b\in\calB\backslash\{\bm{I_{\bar d}}\}$, $\Phi^{*,(b)}=\frac{\inner{\bar\phi^*}{b}}{\inner{b}{b}}$.
\item \textbf{Representation 2.}\quad If $b=\bm{I}_{\bar d}$, then $\Phi^{*,(b)}=1$. For all $b\in\calB\backslash\{\bm{I_{\bar d}}\}$, $\Phi^{*,(b)}=\frac{\inner{\bar\phi^*-\bm{I}_{\bar d}}{b}}{\inner{b}{b}}$.
\end{itemize}
Due to $\calB\backslash\{\bm{I_{\bar d}}\}$ being an orthogonal basis of $\R^{\bar d\times\bar d}$, we have $\bar\phi^*=\sum_{b\in\calB}\Phi^{*,(b)}b$ in both cases. Furthermore, since all $b\in\calB\backslash\{\bm{I_{\bar d}}\}$ can be equivalently expressed as $b=h\otimes e^{(j)}$ for some $h\in\calH$ and $e^{(j)}\in\calE$, we also write the associated $\Phi^{*,(b)}$ equivalently as $\Phi^{*,(h,j)}$. Let $I^{(h)}\subset[1:\bar d]$ be the support of $h$, i.e., the collection of indices $i$ such that the entry $h_i\neq 0$. 

\begin{lemma}\label{lemma:coefficient_magnitude}
For all $\bar\phi^*\in\calS(\bar d)$ and $h\in\calH$,
\begin{equation*}
\sum_{j=1}^{\bar d}\frac{\abs{\inner{\bar\phi^*}{h\otimes e^{(j)}}}}{\inner{h\otimes e^{(j)}}{h\otimes e^{(j)}}}\leq 1,\quad\mathrm{and},\quad \sum_{j=1}^{\bar d}\frac{\abs{\inner{\bar\phi^*-\bm{I}_{\bar d}}{h\otimes e^{(j)}}}}{\inner{h\otimes e^{(j)}}{h\otimes e^{(j)}}}\leq 2.
\end{equation*}
\end{lemma}

That is, $\sum_{j=1}^{\bar d}\abs{\Phi^{*,(h,j)}}$ is at most a constant for both the two representations introduced above. 

\begin{proof}[Proof of Lemma~\ref{lemma:coefficient_magnitude}]
Throughout the proof we write $b=h\otimes e^{(j)}$ for conciseness. By construction $\inner{b}{b}=\inner{h}{h}=\abs{I^{(h)}}$, therefore
\begin{multline*}
\sum_{j=1}^{\bar d}\frac{\abs{\inner{\bar\phi^*}{b}}}{\inner{b}{b}}=\frac{1}{\abs{I^{(h)}}}\sum_{j=1}^{\bar d}\abs{\inner{\bar\phi^*}{b}}=\frac{1}{\abs{I^{(h)}}}\sum_{j=1}^{\bar d}\abs{\sum_{i\in I^{(h)}}\bar\phi^*_{i,j}h_{i}}\leq \frac{1}{\abs{I^{(h)}}}\sum_{j=1}^{\bar d}\sum_{i\in I^{(h)}}\abs{\bar\phi^*_{i,j}h_{i}}\\=\frac{1}{\abs{I^{(h)}}}\sum_{j=1}^{\bar d}\sum_{i\in I^{(h)}}\bar\phi^*_{i,j}
=\frac{1}{\abs{I^{(h)}}}\sum_{i\in I^{(h)}}\rpar{\sum_{j=1}^{\bar d}\bar\phi^*_{i,j}}=1,
\end{multline*}
Since $\bm{I}_{\bar d}\in\calS(\bar d)$, we also have $\sum_{j=1}^{\bar d}\frac{\abs{\inner{\bm{I}_{\bar d}}{b}}}{\inner{b}{b}}\leq 1$, thus
\begin{equation*}
\sum_{j=1}^{\bar d}\frac{\abs{\inner{\bar\phi^*-\bm{I}_{\bar d}}{b}}}{\inner{b}{b}}\leq \sum_{j=1}^{\bar d}\frac{\abs{\inner{\bar\phi^*}{b}}}{\inner{b}{b}}+\sum_{j=1}^{\bar d}\frac{\abs{\inner{\bm{I}_{\bar d}}{b}}}{\inner{b}{b}}\leq 2.\qedhere
\end{equation*}
\end{proof}

The following lemma characterizes the most important property of the Haar wavelet, for our use case. We remark that the $\bar\phi$ matrix below does not have to be a stochastic matrix, and we let $\bar\phi_i$ denote its $i$-th row. 

\begin{lemma}\label{lemma:haar}
Consider an arbitrary $\bar\phi\in\R^{\bar d\times\bar d}$ and $h\in\calH$. If for all $i_1,i_2\in I^{(h)}$ we have $\bar\phi_{i_1}=\bar\phi_{i_2}$, then for all $j\in\bar d$, $\abs{\inner{\bar\phi}{h\otimes e^{(j)}}}=0$.
\end{lemma}

\begin{proof}[Proof of Lemma~\ref{lemma:haar}]
From the condition in the lemma we have $\bar\phi_{i,j};i\in I^{(h)}$ being equal to some $c_j\in\R$ that does not depend on $i$. Therefore,
\begin{equation*}
\abs{\inner{\bar\phi}{h\otimes e^{(j)}}}=\abs{\sum_{i\in I^{(h)}}\bar\phi_{i,j}h_{i}}=\abs{c_j\sum_{i\in I^{(h)}}h_{i}}=\abs{c_j}\abs{\sum_{i\in I^{(h)}}h_{i}}=0.\qedhere
\end{equation*}
\end{proof}

\subsection{Constraint Oracle}

This subsection analyzes our procedure from Section~\ref{subsection:constraint} which enforces the constraint $\calS(\bar d)$. We specifically place our analysis in the context of Algorithm~\ref{algorithm:main}. The following lemma closely mirrors the analysis of \citep{luo2015achieving,orabona2016coin}, and we provide the proof for completeness.

\begin{lemma}[Appendix~D of \citep{orabona2016coin}, adapted]\label{lemma:regret_wrapper}
Regarding the quantities $\bar g_t^\improper$, $\bar \phi_t^\improper$, $\bar g_t$ and $\bar\phi_t$ in Algorithm~\ref{algorithm:main}, we have for all $\bar\phi^*\in\calS(\bar d)$,
\begin{equation*}
\inner{\bar g_t}{\bar \phi_t-\bar \phi^*}\leq \inner{\bar g_t^\improper}{\bar \phi_t^\improper-\bar \phi^*}.
\end{equation*}
\end{lemma}

\begin{proof}[Proof of Lemma~\ref{lemma:regret_wrapper}]
Recall our notation that adding a subscript $i$ to a matrix denotes its $i$-th row, and adding a pair of subscripts $i,j$ to a matrix denotes its ($i,j$)-th entry. Throughout this proof we will also use $\bm{1}$ to represent the all-one vector. 

Starting from the Stage 2 of the wrapper, we have
\begin{align*}
\inner{\bar g_t}{\bar \phi_t-\bar \phi^*}-\inner{\bar g^+_t}{\bar \phi^+_t-\bar \phi^*}&=\sum_{i\in[1:\bar d]}\rpar{\inner{\bar g_{t,i}}{\bar \phi_{t,i}-\bar \phi^*_i}-\inner{\bar g^+_{t,i}}{\bar \phi^+_{t,i}-\bar \phi^*_i}}\\
&=\sum_{i\in[1:\bar d]}\rpar{\inner{\bar g_{t,i}}{\bar \phi_{t,i}-\bar \phi^*_i}-\inner{\bar g_{t,i}-\inner{\bar g_{t,i}}{\bar\phi_{t,i}}\bm{1}}{\bar \phi^+_{t,i}-\bar \phi^*_i}}\tag{Definition of $\bar g^+_{t,i}$}\\
&=\sum_{i\in[1:\bar d]}\rpar{\inner{\bar g_{t,i}}{\bar \phi_{t,i}-\bar \phi^+_{t,i}}+\inner{\bar g_{t,i}}{\bar\phi_{t,i}}\inner{\bm{1}}{\bar \phi^+_{t,i}-\bar \phi^*_i}}.
\end{align*}
Regarding the summand on the RHS, there are two cases. 
\begin{itemize}
\item If $\norm{\bar \phi^+_{t,i}}_1=0$, then since $\bar \phi^*_i\in\Delta(\bar d)$,
\begin{equation*}
\inner{\bar g_{t,i}}{\bar \phi_{t,i}-\bar \phi^+_{t,i}}+\inner{\bar g_{t,i}}{\bar\phi_{t,i}}\inner{\bm{1}}{\bar \phi^+_{t,i}-\bar \phi^*_i}=\inner{\bar g_{t,i}}{\bar \phi_{t,i}}+\inner{\bar g_{t,i}}{\bar\phi_{t,i}}\inner{\bm{1}}{-\bar \phi^*_i}=0.
\end{equation*}
\item If $\norm{\bar \phi^+_{t,i}}_1\neq 0$, then using $\bar \phi_{t,i}=\bar \phi^+_{t,i}/\norm{\bar \phi^+_{t,i}}_1$ and $\bar \phi^+_{t,i}\in\R_+^{\bar d}$,
\begin{align*}
&\inner{\bar g_{t,i}}{\bar \phi_{t,i}-\bar \phi^+_{t,i}}+\inner{\bar g_{t,i}}{\bar\phi_{t,i}}\inner{\bm{1}}{\bar \phi^+_{t,i}-\bar \phi^*_i}\\
=~&\inner{\bar g_{t,i}}{\frac{\bar \phi^+_{t,i}}{\norm{\bar \phi^+_{t,i}}_1}-\bar \phi^+_{t,i}}+\inner{\bar g_{t,i}}{\frac{\bar\phi^+_{t,i}}{\norm{\bar \phi^+_{t,i}}_1}}\inner{\bm{1}}{\bar \phi^+_{t,i}-\bar \phi^*_i}\\
=~&\inner{\bar g_{t,i}}{\frac{\bar \phi^+_{t,i}}{\norm{\bar \phi^+_{t,i}}_1}-\bar \phi^+_{t,i}}+\inner{\bar g_{t,i}}{\frac{\bar\phi^+_{t,i}}{\norm{\bar \phi^+_{t,i}}_1}}\rpar{\norm{\bar \phi^+_{t,i}}_1-1}\\
=~&0.
\end{align*}
\end{itemize}
Therefore $\inner{\bar g_t}{\bar \phi_t-\bar \phi^*}=\inner{\bar g^+_t}{\bar \phi^+_t-\bar \phi^*}$.

Next, consider Stage 1 of the wrapper.
\begin{align*}
&\inner{\bar g^+_t}{\bar \phi^+_t-\bar \phi^*}-\inner{\bar g_t^\improper}{\bar \phi_t^\improper-\bar \phi^*}\\
=~&\sum_{i,j\in[1:\bar d]}\spar{\bar g^+_{t,i,j}(\bar \phi^+_{t,i,j}-\bar \phi^*_{i,j})-\bar g^\improper_{t,i,j}(\bar \phi^\improper_{t,i,j}-\bar \phi^*_{i,j})}\\
=~&\sum_{i,j; \bar \phi^\improper_{t,i,j}<0}\spar{\bar g^+_{t,i,j}(\bar \phi^+_{t,i,j}-\bar \phi^*_{i,j})-\bar g^\improper_{t,i,j}(\bar \phi^\improper_{t,i,j}-\bar \phi^*_{i,j})}\\
=~&\sum_{i,j; \bar \phi^\improper_{t,i,j}<0}\spar{\bar g^+_{t,i,j}(-\bar \phi^*_{i,j})-\min\left\{\bar g^+_{t,i,j},0\right\}(\bar \phi^\improper_{t,i,j}-\bar \phi^*_{i,j})}\\
=~&\sum_{i,j; \bar \phi^\improper_{t,i,j}<0}\spar{-\min\left\{\bar g^+_{t,i,j},0\right\}\bar \phi^\improper_{t,i,j}+\bar \phi^*_{i,j}(\min\left\{\bar g^+_{t,i,j},0\right\}-\bar g^+_{t,i,j})}\\
\leq~& 0.
\end{align*}
Combining it with the result of Stage 2 above completes the proof.
\end{proof}

In the next lemma, for all $b\in \calB\backslash\{\bm{I_{\bar d}}\}$ we write it equivalently as $b=h\otimes e^{(j)}$ for some $h\in\calH$ and $e^{(j)}\in\calE$. Let $I^{(h)}\subset[1:\bar d]$ be the support of the vector $h$. We emphasize that the lemma hinges on the structure of the Haar wavelet. 

\begin{lemma}\label{lemma:gradient_norm}
With $b=h\otimes e^{(j)}$, the quantity $\inner{\bar g_t^\improper}{b}$ in Algorithm~\ref{algorithm:main} satisfies
\begin{equation*}
\abs{\inner{\bar g_t^\improper}{b}}\leq 2\sum_{i\in I^{(h)}}\bar p_{t,i}.
\end{equation*}
Furthermore, $\abs{\inner{\bar g_t^\improper}{\bm{I}_{\bar d}}}\leq 2$.
\end{lemma}

\begin{proof}[Proof of Lemma~\ref{lemma:gradient_norm}]
First, consider the general case of $b=h\otimes e^{(j)}$. 
\begin{equation*}
\abs{\inner{\bar g_t^\improper}{b}}=\abs{\inner{\bar g_t^\improper}{h\otimes e^{(j)}}}=\abs{\sum_{i\in I^{(h)}}\bar g^\improper_{t,i,j}h_{i}}
\leq \sum_{i\in I^{(h)}}\abs{\bar g^\improper_{t,i,j}h_{i}}=\sum_{i\in I^{(h)}}\abs{\bar g^\improper_{t,i,j}},
\end{equation*}
where the last equality is due to the entries of $h$ being $\pm 1$ on its support. From Stage 1 of the gradient processing oracle, we have $\abs{\bar g^\improper_{t,i,j}}$ being either $\abs{\bar g^+_{t,i,j}}$ or $0$. Therefore,
\begin{equation*}
\sum_{i\in I^{(h)}}\abs{\bar g^\improper_{t,i,j}}\leq \sum_{i\in I^{(h)}}\abs{\bar g^+_{t,i,j}}.
\end{equation*}
From Stage 2 of the gradient processing oracle, and using $\bar g_t=\bar p_t\otimes \bar l_t$,
\begin{equation*}
\sum_{i\in I^{(h)}}\abs{\bar g^+_{t,i,j}}=\sum_{i\in I^{(h)}}\abs{\bar g_{t,i,j}-\inner{\bar g_{t,i}}{\bar\phi_{t,i}}}=\sum_{i\in I^{(h)}}\bar p_{t,i}\abs{\bar l_{t,j}-\inner{\bar l_t}{\bar\phi_{t,i}}}\leq 2\sum_{i\in I^{(h)}}\bar p_{t,i}.
\end{equation*}

As for the case of $b=\bm{I}_{\bar d}$, 
\begin{multline*}
\abs{\inner{\bar g_t^\improper}{\bm{I}_{\bar d}}}=\abs{\sum_{i\in [1:\bar d]}\bar g^\improper_{t,i,i}}\leq \sum_{i\in [1:\bar d]}\abs{\bar g^\improper_{t,i,i}}\leq \sum_{i\in [1:\bar d]}\abs{\bar g^+_{t,i,i}}\\
=\sum_{i\in [1:\bar d]}\abs{\bar g_{t,i,i}-\inner{\bar g_{t,i}}{\bar\phi_{t,i}}}=\sum_{i\in [1:\bar d]}\bar p_{t,i}\abs{\bar l_{t,i}-\inner{\bar l_{t}}{\bar\phi_{t,i}}}\leq 2\sum_{i\in [1:\bar d]}\bar p_{t,i}=2.\qedhere
\end{multline*}
\end{proof}

\section{Proof of Main Results}\label{section:proof_main}

Our main result (Theorem~\ref{theorem:main}) is a combination of Theorem~\ref{theorem:external} and \ref{theorem:internal} proved below. 

\begin{theorem}\label{theorem:external}
For any comparing action modification rule $\phi^*\in\calS(d)$, let $\bar\phi^*\in\calS(\bar d)$ be the construction from Lemma~\ref{lemma:relabeling_uniform}. Then, there is an absolute constant $c>0$ such that for all $T\in\N_+$ and $\phi^*\in\calS(d)$, Algorithm~\ref{algorithm:main} guarantees
\begin{equation*}
\reg_T(\phi^*)\leq c\cdot\rpar{\sqrt{(T+d)\rpar{1+\sum_{i=1}^{\bar d-1}\bm{1}[\bar \phi^*_{i}\neq \bar \phi^*_{i+1}]}}\cdot \rpar{\log d}^{3/2}}.
\end{equation*}
\end{theorem}

\begin{proof}[Proof of Theorem~\ref{theorem:external}]
We structure the proof into the following three steps. 

\paragraph{Step 1} Putting together the guarantees of individual components, arriving at the combined regret bound, Eq.(\ref{eq:big_regret}). 

Due to Lemma~\ref{lemma:relabeling_uniform}, $\inner{p_t-\phi^*(p_t)}{l_t}=\inner{\bar p_t-\bar \phi^*(\bar p_t)}{\bar l_t}$, therefore regarding our objective we have
\begin{equation}\label{eq:regret_first_step}
\reg_T(\phi^*)=\sum_{t=1}^T\inner{\bar p_t-\bar \phi^*(\bar p_t)}{\bar l_t}. 
\end{equation}

Next, for all $b\in\calB$ we consider $\Phi^{*,(b)}\in\R$ defined according to Representation 1 in Appendix~\ref{subsection:detail_matrix_feature}: if $b=\bm{I}_{\bar d}$, then $\Phi^{*,(b)}=0$; for all $b\in\calB\backslash\{\bm{I_{\bar d}}\}$, $\Phi^{*,(b)}=\frac{\inner{\bar\phi^*}{b}}{\inner{b}{b}}$. This ensures $\bar\phi^*=\sum_{b\in\calB}\Phi^{*,(b)}b$. In addition, since all $b\in\calB\backslash\{\bm{I}_{\bar d}\}$ can be expressed as $b=h\otimes e^{(j)}$ for some $h\in\calH$ and $e^{(j)}\in\calE$, we will write their corresponding $\Phi^{*,(b)}$ equivalently as $\Phi^{*,(h,j)}$. With any fixed $h$, $\sum_{j=1}^{\bar d}\abs{\Phi^{*,(h,j)}}\leq 1$ due to Lemma~\ref{lemma:coefficient_magnitude}. 

Now consider an arbitrary $b\in\calB$. Due to Lemma~\ref{lemma:gradient_norm}, the one-dimensional gradients that Algorithm~\ref{algorithm:main} sends to $\calA_\oned^{(b)}$ satisfy $\abs{\inner{\bar g_t^\improper}{b}}\leq 2$, which means Lemma~\ref{lemma:one_d_alg} can be applied with $G=2$. This yields the one-dimensional regret bound with respect to $\Phi^{*,(b)}$,
\begin{multline*}
\sum_{t=1}^T\inner{\bar g_t^\improper}{b}(\Phi_t^{(b)}-\Phi^{*,(b)})\\
\leq \sqrt{8+4\sum_{t=1}^T\abs{\inner{\bar g_t^\improper}{b}}}\spar{\sqrt{2}\eps_b+2\abs{\Phi^{*,(b)}}\sqrt{\log\rpar{1+\frac{\abs{\Phi^{*,(b)}}}{\sqrt{2}\eps_b}}}+4\abs{\Phi^{*,(b)}}}.
\end{multline*}
Taking a summation over $b\in\calB$, the LHS becomes
\begin{align*}
\sum_{b\in\calB}\sum_{t=1}^T\inner{\bar g_t^\improper}{b}(\Phi_t^{(b)}-\Phi^{*,(b)})&=\sum_{t=1}^T\inner{\bar g_t^\improper}{\sum_{b\in\calB}(\Phi_t^{(b)}-\Phi^{*,(b)})b}\\
&=\sum_{t=1}^T\inner{\bar g_t^\improper}{\bar\phi_t^{\mathrm{improper}}-\bar\phi^*}\\
&\geq \sum_{t=1}^T\inner{\bar g_t}{\bar \phi_t-\bar \phi^*}\tag{Lemma~\ref{lemma:regret_wrapper}}\\
&=\sum_{t=1}^T\inner{\bar p_t-\bar\phi^*(\bar p_t)}{\bar l_t}\tag{$\phi$-to-external reduction}\\
&=\reg_T(\phi^*).\tag{Eq.(\ref{eq:regret_first_step})}
\end{align*}
Therefore, 
\begin{align}
\reg_T(\phi^*)\leq~& \sum_{b\in\calB}\sqrt{8+4\sum_{t=1}^T\abs{\inner{\bar g_t^\improper}{b}}}\spar{\sqrt{2}\eps_b+2\abs{\Phi^{*,(b)}}\sqrt{\log\rpar{1+\frac{\abs{\Phi^{*,(b)}}}{\sqrt{2}\eps_b}}}+4\abs{\Phi^{*,(b)}}}\nonumber\\
=~&\sqrt{16+8\sum_{t=1}^T\abs{\inner{\bar g_t^\improper}{\bm{I}_{\bar d}}}}\nonumber\\
&\quad+\sum_{j=1}^{\bar d}\sqrt{8+4\sum_{t=1}^T\abs{\inner{\bar g_t^\improper}{h^{(S,0)}\otimes e^{(j)}}}}\Bigg[\frac{\sqrt{2}\bar d}{d^2\abs{\calI(\calI^{-1}(j))}}\nonumber\\
&\quad\quad +2\abs{\Phi^{*,(h^{(S,0)},j)}}\sqrt{\log\rpar{1+\frac{\abs{\Phi^{*,(h^{(S,0)},j)}}}{\sqrt{2}\bar d(d^2\abs{\calI(\calI^{-1}(j))})^{-1}}}}+4\abs{\Phi^{*,(h^{(S,0)},j)}}\Bigg]\nonumber\\
&\quad\quad\quad+\sum_{h\in\calH\backslash\{h^{(S,0)}\}}\sum_{j=1}^{\bar d}\sqrt{8+4\sum_{t=1}^T\abs{\inner{\bar g_t^\improper}{h\otimes e^{(j)}}}}\Bigg[\frac{\sqrt{2}}{d^2\abs{\calI(\calI^{-1}(j))}}\nonumber\\
&\quad\quad\quad\quad +2\abs{\Phi^{*,(h,j)}}\sqrt{\log\rpar{1+\frac{\abs{\Phi^{*,(h,j)}}}{\sqrt{2}(d^2\abs{\calI(\calI^{-1}(j))})^{-1}}}}+4\abs{\Phi^{*,(h,j)}}\Bigg].\label{eq:big_regret}
\end{align}

\paragraph{Step 2} Separately bounding the three parts in Eq.(\ref{eq:big_regret}), corresponding to three different types of $b$. 

We now separately bound the three terms on the RHS of Eq.(\ref{eq:big_regret}). Denote them as $\mathrm{PartOne}$ (the first line), $\mathrm{PartTwo}$ (the second and third line), and $\mathrm{PartThree}$ (the last two lines). 
\begin{itemize}
\item Due to Lemma~\ref{lemma:gradient_norm}, $\abs{\inner{\bar g_t^\improper}{\bm{I}_{\bar d}}}\leq 2$, therefore
\begin{equation*}
\mathrm{PartOne}=\sqrt{16+8\sum_{t=1}^T\abs{\inner{\bar g_t^\improper}{\bm{I}_{\bar d}}}}\leq 4\sqrt{T+1}.
\end{equation*}
\item Due to Lemma~\ref{lemma:gradient_norm}, $\abs{\inner{\bar g_t^\improper}{h^{(S,0)}\otimes e^{(j)}}}\leq 2\sum_{i=1}^{\bar d}\bar p_{t,i}=2$, therefore,
\begin{align}
\mathrm{PartTwo}
&\leq 2\sqrt{2}\sqrt{T+1}\sum_{j=1}^{\bar d}\Biggl[\frac{\sqrt{2}\bar d}{d^2\abs{\calI(\calI^{-1}(j))}}\nonumber\\
&\quad+2\abs{\Phi^{*,(h^{(S,0)},j)}}\sqrt{\log\rpar{1+\frac{\abs{\Phi^{*,(h^{(S,0)},j)}}}{\sqrt{2}\bar d(d^2\abs{\calI(\calI^{-1}(j))})^{-1}}}}+4\abs{\Phi^{*,(h^{(S,0)},j)}}\Biggr]\nonumber\\
&= 2\sqrt{2}\sqrt{T+1}\Bigg[\sqrt{2}\bar d d^{-1}\nonumber\\
&\quad+2\sum_{j=1}^{\bar d}\abs{\Phi^{*,(h^{(S,0)},j)}}\sqrt{\log\rpar{1+\frac{\abs{\Phi^{*,(h^{(S,0)},j)}}}{\sqrt{2}\bar d(d^2\abs{\calI(\calI^{-1}(j))})^{-1}}}}+4\sum_{j=1}^{\bar d}\abs{\Phi^{*,(h^{(S,0)},j)}}\Bigg]\nonumber\\
&\leq 2\sqrt{2}\sqrt{T+1}\spar{2\sqrt{2}+4+2\sum_{j=1}^{\bar d}\abs{\Phi^{*,(h^{(S,0)},j)}}\sqrt{\log\rpar{1+\frac{\abs{\Phi^{*,(h^{(S,0)},j)}}}{\sqrt{2}\bar d(d^2\abs{\calI(\calI^{-1}(j))})^{-1}}}}}.\label{eq:termtwo}
\end{align}
For now we use a crude bound: $d\leq \bar d$, $\abs{\calI(\calI^{-1}(j))}\leq 2$ and $\abs{\Phi^{*,(h^{(S,0)},j)}}\leq 1$, therefore
\begin{align*}
\mathrm{PartTwo}&\leq 2\sqrt{2}\sqrt{T+1}\spar{2\sqrt{2}+4+2\sum_{j=1}^{\bar d}\abs{\Phi^{*,(h^{(S,0)},j)}}\sqrt{\log\rpar{1+\sqrt{2}d}}}\\
&\leq 2\sqrt{2}\sqrt{T+1}\spar{2\sqrt{2}+4+2\sqrt{\log\rpar{1+\sqrt{2}d}}}\\
&=O(\sqrt{T\log d}). 
\end{align*}
Later on we will return to Eq.(\ref{eq:termtwo}) when analyzing the quantile regret. 
\item Due to Lemma~\ref{lemma:gradient_norm}, $\abs{\inner{\bar g_t^\improper}{h\otimes e^{(j)}}}\leq 2\sum_{i\in I^{(h)}}\bar p_{t,i}$ for all $h\in\calH$. We also have $\abs{\calI(\calI^{-1}(j))}\leq 2$ and $\abs{\Phi^{*,(h,j)}}\leq 1$. Therefore similar to the above, 
\begin{align*}
&\mathrm{PartThree}\\
&\leq 2\sqrt{2}\sum_{h\in\calH\backslash\{h^{(S,0)}\}}\sum_{j=1}^{\bar d}\sqrt{1+\sum_{t=1}^T\sum_{i\in I^{(h)}}\bar p_{t,i}}\spar{\frac{\sqrt{2}}{d^2\abs{\calI(\calI^{-1}(j))}}+2\abs{\Phi^{*,(h,j)}}\rpar{\sqrt{\log\rpar{1+\sqrt{2}d^2}}+2}}\\
&= 2\sqrt{2}\sum_{h\in\calH\backslash\{h^{(S,0)}\}}\sqrt{1+\sum_{t=1}^T\sum_{i\in I^{(h)}}\bar p_{t,i}}\spar{\frac{\sqrt{2}}{d}+2\sum_{j=1}^{\bar d}\abs{\Phi^{*,(h,j)}}\rpar{\sqrt{\log\rpar{1+\sqrt{2}d^2}}+2}}\\
&\leq 4\sqrt{T+1}+4\sqrt{2}\rpar{\sqrt{\log\rpar{1+\sqrt{2}d^2}}+2}\underbrace{\sum_{h\in\calH\backslash\{h^{(S,0)}\}}\rpar{\sum_{j=1}^{\bar d}\abs{\Phi^{*,(h,j)}}}\sqrt{1+\sum_{t=1}^T\sum_{i\in I^{(h)}}\bar p_{t,i}}}_{\eqdef\Diamond}.
\end{align*}
The key step of our analysis is bounding this $\Diamond$ term. 
\end{itemize}

\paragraph{Step 3} Exploiting sparsity. 

To this end, notice that for all $h\in\calH\backslash\{h^{(S,0)}\}$, we can write it as $h^{(s,l)}$ for some scale parameter $s\in[1:S]$ and some location parameter $l\in[1:2^{S-s}]$. Accordingly, we will write $\Phi^{*,(h,j)}$ equivalently as $\Phi^{*,(s,l,j)}$. Then, 
\begin{align*}
\Diamond&=\sum_{s=1}^S\sum_{l=1}^{2^{S-s}}\rpar{\sum_{j=1}^{\bar d}\abs{\Phi^{*,(s,l,j)}}}\sqrt{1+\sum_{t=1}^T\sum_{i\in I^{(s,l)}}\bar p_{t,i}}\\
&=\sum_{s=1}^S\sum_{l=1}^{2^{S-s}}\rpar{\sum_{j=1}^{\bar d}\frac{\abs{\inner{\bar\phi^*}{h^{(s,l)}\otimes e^{(j)}}}}{\inner{h^{(s,l)}\otimes e^{(j)}}{h^{(s,l)}\otimes e^{(j)}}}}\sqrt{1+\sum_{t=1}^T\sum_{i\in I^{(s,l)}}\bar p_{t,i}}\\
&=\sum_{s=1}^S\sum_{l=1}^{2^{S-s}}\rpar{\frac{1}{\abs{I^{(s,l)}}}\sum_{j=1}^{\bar d}\abs{\inner{\bar\phi^*}{h^{(s,l)}\otimes e^{(j)}}}}\sqrt{1+\sum_{t=1}^T\sum_{i\in I^{(s,l)}}\bar p_{t,i}}.
\end{align*}

Now consider fixing $s$ and letting $l$ vary. Due to Lemma~\ref{lemma:haar}, if $\bar\phi^*$ remains the same for all its rows with indices in $I^{(s,l)}$, then $\sum_{j=1}^{\bar d}\abs{\inner{\bar\phi^*}{h^{(s,l)}\otimes e^{(j)}}}=0$. Since the number of row-changes is $\sum_{i=1}^{\bar d-1}\bm{1}[\bar \phi^*_{i}\neq \bar \phi^*_{i+1}]$, there are at most this amount of $l$ such that $\sum_{j=1}^{\bar d}\abs{\inner{\bar\phi^*}{h^{(s,l)}\otimes e^{(j)}}}\neq 0$. Therefore using Cauchy-Schwarz, 
\begin{align*}
&\sum_{l=1}^{2^{S-s}}\rpar{\frac{1}{\abs{I^{(s,l)}}}\sum_{j=1}^{\bar d}\abs{\inner{\bar\phi^*}{h^{(s,l)}\otimes e^{(j)}}}}\sqrt{1+\sum_{t=1}^T\sum_{i\in I^{(s,l)}}\bar p_{t,i}}\\
=~&\sum_{l;l\in[1:2^{S-s}],\sum_{j=1}^{\bar d}\abs{\inner{\bar\phi^*}{h^{(s,l)}\otimes e^{(j)}}}\neq 0}\rpar{\frac{1}{\abs{I^{(s,l)}}}\sum_{j=1}^{\bar d}\abs{\inner{\bar\phi^*}{h^{(s,l)}\otimes e^{(j)}}}}\sqrt{1+\sum_{t=1}^T\sum_{i\in I^{(s,l)}}\bar p_{t,i}}\\
\leq~&\sqrt{\abs{\left\{l;l\in[1:2^{S-s}],\sum_{j=1}^{\bar d}\abs{\inner{\bar\phi^*}{h^{(s,l)}\otimes e^{(j)}}}\neq 0\right\}}}\\
&\quad\cdot\sqrt{\sum_{\substack{l:l\in[1:2^{S-s}],\\\sum_{j=1}^{\bar d}\abs{\inner{\bar\phi^*}{h^{(s,l)}\otimes e^{(j)}}}\neq 0}}\rpar{\frac{1}{\abs{I^{(s,l)}}}\sum_{j=1}^{\bar d}\abs{\inner{\bar\phi^*}{h^{(s,l)}\otimes e^{(j)}}}}^2\rpar{1+\sum_{t=1}^T\sum_{i\in I^{(s,l)}}\bar p_{t,i}}}\\
\leq~& \sqrt{\sum_{i=1}^{\bar d-1}\bm{1}[\bar \phi^*_{i}\neq \bar \phi^*_{i+1}]}\sqrt{\sum_{l=1}^{2^{S-s}}\rpar{\frac{1}{\abs{I^{(s,l)}}}\sum_{j=1}^{\bar d}\abs{\inner{\bar\phi^*}{h^{(s,l)}\otimes e^{(j)}}}}^2\rpar{1+\sum_{t=1}^T\sum_{i\in I^{(s,l)}}\bar p_{t,i}}}\\
\leq~& \sqrt{\sum_{i=1}^{\bar d-1}\bm{1}[\bar \phi^*_{i}\neq \bar \phi^*_{i+1}]}\sqrt{\sum_{l=1}^{2^{S-s}}\rpar{1+\sum_{t=1}^T\sum_{i\in I^{(s,l)}}\bar p_{t,i}}}\tag{Lemma~\ref{lemma:coefficient_magnitude}}\\
\leq~& \sqrt{\sum_{i=1}^{\bar d-1}\bm{1}[\bar \phi^*_{i}\neq \bar \phi^*_{i+1}]}\sqrt{\bar d+\sum_{t=1}^T\sum_{l=1}^{2^{S-s}}\sum_{i\in I^{(s,l)}}\bar p_{t,i}}\\
\leq~& \sqrt{\sum_{i=1}^{\bar d-1}\bm{1}[\bar \phi^*_{i}\neq \bar \phi^*_{i+1}]}\sqrt{2d+T}.
\end{align*}
Since this holds for all $s\in[1:S]$ and $S=O(\log d)$, overall we have
\begin{equation*}
\Diamond=O\rpar{\sqrt{(T+d)\sum_{i=1}^{\bar d-1}\bm{1}[\bar \phi^*_{i}\neq \bar \phi^*_{i+1}]}\cdot \log d}.
\end{equation*}
Combining everything above completes the proof.
\end{proof}

\begin{theorem}\label{theorem:internal}
There is an absolute constant $c>0$ such that for all $T\in\N_+$ and $\phi^*\in\calS(d)$, Algorithm~\ref{algorithm:main} guarantees
\begin{equation*}
\reg_T(\phi^*)\leq c\cdot\rpar{\sqrt{(d-d^\self_{\phi^*})(T+d)}\cdot(\log d)^{3/2}}.
\end{equation*}
\end{theorem}

\begin{proof}[Proof of Theorem~\ref{theorem:internal}]
We start by following the same proof as Theorem~\ref{theorem:external}, with only two differences: 
\begin{itemize}
\item For any $\phi^*\in\calS(d)$, we let $\bar\phi^*\in\calS(\bar d)$ be the construction from Lemma~\ref{lemma:relabeling_self}.
\item $\Phi^{*,(b)}$ is defined according to Representation 2 in Appendix~\ref{subsection:detail_matrix_feature}. That is, if $b=\bm{I}_{\bar d}$, then $\Phi^{*,(b)}=1$; for all $b\in\calB\backslash\{\bm{I_{\bar d}}\}$, $\Phi^{*,(b)}=\frac{\inner{\bar\phi^*-\bm{I}_{\bar d}}{b}}{\inner{b}{b}}$.
\end{itemize}
Then, analogous to Eq.(\ref{eq:big_regret}), we take a summation of the one-dimensional regret over all matrix features $b\in\calB$ and obtain 
\begin{equation*}
\reg_T(\phi^*)\leq \mathrm{PartOne}+\mathrm{PartTwo}+\mathrm{PartThree},
\end{equation*}
where each part on the RHS represents the summation over a certain subset of $b\in\calB$. 

Specifically, only including $b=\bm{I}_{\bar d}$ we have $\mathrm{PartOne}=O(\sqrt{T})$. For the summation over all $b=h^{(S,0)}\otimes e^{(j)}$; $\forall j$, we have $\mathrm{PartTwo}=O(\sqrt{T\log d})$. Finally, for the summation over the rest of the $b\in\calB$, we have
\begin{align*}
\mathrm{PartThree}\leq O(\sqrt{T})+O(\sqrt{\log d})\underbrace{\sum_{h\in\calH\backslash\{h^{(S,0)}\}}\rpar{\sum_{j=1}^{\bar d}\abs{\Phi^{*,(h,j)}}}\sqrt{1+\sum_{t=1}^T\sum_{i\in I^{(h)}}\bar p_{t,i}}}_{\eqdef\Diamond}.
\end{align*}
Everything here is only different from the corresponding term in the proof of Theorem~\ref{theorem:external} by constant multiplying factors. 

The analysis of $\Diamond$ is also similar to Step 3 in the proof of Theorem~\ref{theorem:external}, with $\bar\phi^*$ replaced by $\bar\phi^*-\bm{I}_{\bar d}$ (since in this proof when $b=\bm{I}_{\bar d}$ we define $\Phi^{*,(b)}=1$, rather than $\Phi^{*,(b)}=0$ as in the proof of Theorem~\ref{theorem:external}). By the definition of $d^\self_{\bar\phi^*}$, only $\bar d-d^\self_{\bar\phi^*}$ rows of $\bar\phi^*-\bm{I}_{\bar d}$ are nonzero. Accordingly, there are at most $2(\bar d-d^\self_{\bar\phi^*})$ row-changes, therefore
\begin{equation*}
\Diamond=O\rpar{\sqrt{\rpar{\bar d-d^\self_{\bar\phi^*}}(T+d)}\cdot \log d}.
\end{equation*}
Due to Lemma~\ref{lemma:relabeling_self}, $\bar d-d^\self_{\bar\phi^*}$ here can be further replaced by $d-d^\self_{\phi^*}$. 

Combining the above proves the theorem in the case of $d^\self_{\bar\phi^*}\leq d-1$. Finally observe that $\reg_T(\phi^*)$ trivially equals $0$ when $d^\self_{\bar\phi^*}=d$ (i.e., $\phi^*=\bm{I}_d$), therefore the regret bound in the theorem holds in that case as well. 
\end{proof}

Next, we restate and prove the quantile regret bound of Algorithm~\ref{algorithm:main}. 

\quantile*

\begin{proof}[Proof of Theorem~\ref{theorem:quantile}]
We start by considering the $\phi$-regret $\reg_T(\phi^*)$, where $\phi^*=\bm{1}\otimes u\in\calS(d)$ for some $u\in\Delta(d)$ (here $\bm{1}$ denotes the length-$d$ all-one vector). Following exactly the proof of Theorem~\ref{theorem:external} until Eq.(\ref{eq:big_regret}), we have
\begin{equation*}
\reg_T(\phi^*)\leq \mathrm{PartOne}+\mathrm{PartTwo}+\mathrm{PartThree},
\end{equation*}
where $\mathrm{PartOne}\leq 4\sqrt{T+1}$, and $\mathrm{PartThree}\leq 4\sqrt{T+1}$ due to $\Diamond=0$ (the definition of $\Diamond$ is the same as in the proof of Theorem~\ref{theorem:external}). It only remains to analyze $\mathrm{PartTwo}$. 

Now consider $\Phi^{*,(h^{(S,0)},j)}$ that appears in $\mathrm{PartTwo}$. 
\begin{equation*}
\Phi^{*,(h^{(S,0)},j)}=\frac{\inner{\bar\phi^*}{h^{(S,0)}\otimes e^{(j)}}}{\inner{h^{(S,0)}\otimes e^{(j)}}{h^{(S,0)}\otimes e^{(j)}}}=d^{-1}\sum_{i=1}^{\bar d}\bar\phi^*_{i,j},
\end{equation*}
where due to the construction of Lemma~\ref{lemma:relabeling_uniform}, $\bar\phi^*_{i,j}=\frac{\phi^*_{\calI^{-1}(i),\calI^{-1}(j)}}{\abs{\calI(\calI^{-1}(j))}}$. Therefore, 
\begin{equation*}
\Phi^{*,(h^{(S,0)},j)}=d^{-1}\sum_{i=1}^{\bar d}\frac{\phi^*_{\calI^{-1}(i),\calI^{-1}(j)}}{\abs{\calI(\calI^{-1}(j))}}=\bar dd^{-1}\frac{u_{\calI^{-1}(j)}}{\abs{\calI(\calI^{-1}(j))}}.
\end{equation*}

Continuing from Eq.(\ref{eq:termtwo}),
\begin{equation*}
\mathrm{PartTwo}\leq 2\sqrt{2}\sqrt{T+1}\spar{2\sqrt{2}+4+2\sum_{j=1}^{\bar d}\abs{\Phi^{*,(h^{(S,0)},j)}}\sqrt{\log\rpar{1+\frac{\abs{\Phi^{*,(h^{(S,0)},j)}}}{\sqrt{2}\bar d(d^2\abs{\calI(\calI^{-1}(j))})^{-1}}}}},
\end{equation*}
where
\begin{align*}
\sum_{j=1}^{\bar d}\abs{\Phi^{*,(h^{(S,0)},j)}}\sqrt{\log\rpar{1+\frac{\abs{\Phi^{*,(h^{(S,0)},j)}}}{\sqrt{2}(d\abs{\calI(\calI^{-1}(j))})^{-1}}}}
=~&\bar dd^{-1}\sum_{j=1}^{\bar d}\frac{u_{\calI^{-1}(j)}}{\abs{\calI(\calI^{-1}(j))}}\sqrt{\log\rpar{1+d\frac{u_{\calI^{-1}(j)}}{\sqrt{2}}}}\\
=~&\bar dd^{-1}\sum_{\calI^{-1}(j)=1}^{d}u_{\calI^{-1}(j)}\sqrt{\log\rpar{1+d\frac{u_{\calI^{-1}(j)}}{\sqrt{2}}}}\\
\leq~& 2\sum_{i=1}^{d}u_{i}\sqrt{\log\rpar{1+\frac{u_{i}}{\sqrt{2}d^{-1}}}}.
\end{align*}
Putting things together, 
\begin{equation*}
\reg_T(\phi^*)=O\rpar{\sqrt{T}\sum_{i=1}^{d}u_{i}\sqrt{\log\rpar{1+\frac{u_{i}}{\sqrt{2}d^{-1}}}}}.
\end{equation*}

From this, bounding the quantile regret $\reg_T(\eps)$ follows from the standard trick \citep[Remark~10.12]{orabona2025modern}. $\reg_T(\eps)=\reg_T(\phi^*)$ when $u\in\Delta(d)$ is defined entrywise as follows: if the $i$-th expert is among the $\ceil{\eps d}$ best experts (denoted as the index set $I^*$), then $u_i=1/\ceil{\eps d}$; otherwise $u_i=0$. Then, 
\begin{equation*}
\reg_T(\eps)=O\rpar{\sqrt{T}\sum_{i\in I^*}\frac{1}{\ceil{\eps d}}\sqrt{\log\rpar{1+\frac{d}{\sqrt{2}\ceil{\eps d}}}}}=O\rpar{\sqrt{T\log \eps^{-1}}}.\qedhere
\end{equation*}
\end{proof}

\section{Discussion}\label{section:discussion}

In this section we discuss the strength of our result in the regime of large $d^\self_{\phi^*}$, complementing the regime of large $d^\unif_{\phi^*}$ discussed in Section~\ref{section:main}. 

Let us consider $\reg_T(\phi^*)$ where the comparing action modification rule $\phi^*$ satisfies $d^\self_{\phi^*}=d-k$. In other words, $\phi^*$ only nontrivially modifies $k$ of the $d$ experts; the rest are kept unchanged. Due to Theorem~\ref{theorem:main}, Algorithm~\ref{algorithm:main} guarantees $\reg_T(\phi^*)=\tilde O(\sqrt{kT})$ while being computationally efficient and agnostic to $k$. This improves
\begin{itemize}
\item the $\tilde O(\sqrt{dT})$ bound achieved by generic swap regret minimization \citep{blum2007external};
\item the $\tilde O(k\sqrt{T})$ bound achieved by standard internal regret minimization \citep[Chapter~4.4]{cesa2006prediction}, since the considered $\reg_T(\phi^*)$ is at most $k$ times the internal regret; and
\item the computationally inefficient and $k$-dependent approach to achieve the $\tilde O(\sqrt{kT})$ bound, which runs MWU over all zero-one stochastic matrices $\phi$ satisfying $d^\self_{\phi}=d-k$. 
\end{itemize}

Additionally, we devote special attention to the following baseline. This is a fairly straightforward corollary of \citep[Section~3.3]{roth2023learning}, despite not being formally stated there. The baseline concerns internal regret minimization: following \citep[Definition~32]{roth2023learning}, it formulates the internal regret as a special case of the subsequence regret \citep[Definition~28 and 29]{roth2023learning}, which is then handled by a generic reduction from \citep[Theorem~14]{roth2023learning}. To be more specific, this reduction requires a ``meta'' LEA algorithm as input,\footnote{Intuitively, the ``meta'' LEA algorithm treats each subsequence as a ``meta'' expert, and reweighs their importance according to their performance (i.e., subsequence regret). The higher the subsequence regret is, the larger weight the corresponding ``meta'' expert would receive. This ``scalarizes'' the multi-objective problem of minimizing all subsequence regret to a single-objective problem of minimizing the weighted sum of these regrets.} and \citep[Theorem~14]{roth2023learning} showed that the associated subsequence regret can be bounded by the coordinate-wise external regret of the input ``meta'' LEA algorithm. \cite{noarov2023high} further suggested initiating this reduction with a ``coordinate-wise, second-order adaptive'' LEA algorithm from \citep{chen2021impossible}. Combining everything, this yields a computationally efficient algorithm for our setting, whose internal regret bound is
\begin{equation*}
\sum_{t=1}^Tp_{t,i}(l_{t,i}-l_{t,j})=O\rpar{\sqrt{\sum_{t=1}^Tp^2_{t,i}}\cdot \log (dT)},\quad\forall i,j\in[1:d].
\end{equation*}
Here, $p_{t,i}$ denotes the $i$-th coordinate of the prediction $p_t\in\Delta(d)$, and $l_{t,i},l_{t,j}$ denote the $i$-th and the $j$-th coordinate of the loss vector $l_t$. This result sharpens the standard $O(\sqrt{T\log d})$ internal regret bound of \citep[Chapter~4.4]{cesa2006prediction}, modulo log factors. 

Now, consider $\reg_T(\phi^*)$ again with $d^\self_{\phi^*}=d-k$. Without knowing $k$, the above internal regret minimization baseline guarantees
\begin{align*}
\reg_T(\phi^*)&=O\rpar{\sum_{i;\phi^*(e^{(i)})\neq e^{(i)}}\sqrt{\sum_{t=1}^Tp^2_{t,i}}\cdot \log (dT)}\\
&=O\rpar{\sqrt{k}\cdot \sqrt{\sum_{i;\phi^*(e^{(i)})\neq e^{(i)}}\sum_{t=1}^Tp^2_{t,i}}\cdot \log (dT)}\tag{Cauchy-Schwarz}\\
&=O\rpar{\sqrt{k}\cdot \sqrt{\sum_{t=1}^T\sum_{i=1}^d p^2_{t,i}}\cdot \log (dT)}\\
&=O\rpar{\sqrt{kT}\cdot \log (dT)}.
\end{align*}
Although also being $\tilde O(\sqrt{kT})$, its difference with our result is the log factor: the above log factor depends on $T$, whereas the log factor in our Theorem~\ref{theorem:main} does not. Recall that in our setting, $T$ can be arbitrarily larger than $d$. 

Finally, we remark that the last line of the above bound shows the second-order adaptivity of \cite{chen2021impossible} is a slight overkill. Using the ``coordinate-wise first-order adaptive'' LEA algorithm of \cite{luo2015achieving} would guarantee $\reg_T(\phi^*)=O\rpar{\sqrt{kT}\cdot(\log d + \log\log T)}$, but we are not aware of any existing approach that completely removes the $T$-dependence of the log factor. 

\end{document}